\documentclass{article}

\usepackage{arxiv}

\usepackage[utf8]{inputenc}
\usepackage[T1]{fontenc}
\usepackage{hyperref}
\usepackage{url}
\usepackage{booktabs}
\usepackage{amsfonts}
\usepackage{microtype}
\usepackage{graphicx}
\usepackage{xcolor}
\usepackage{subcaption}
\usepackage{float}
\usepackage{placeins}
\usepackage{amsmath}
\usepackage{amssymb}
\usepackage{mathtools}
\usepackage{amsthm}
\usepackage{natbib}
\usepackage[capitalize,noabbrev]{cleveref}

\setkeys{Gin}{draft=false}
\definecolor{darkblue}{rgb}{0,0,0.5}
\hypersetup{colorlinks=true, citecolor=darkblue, linkcolor=darkblue, urlcolor=darkblue}
\pdftrailerid{}

\theoremstyle{plain}
\newtheorem{theorem}{Theorem}[section]

\newtheorem{lemma}[theorem]{Lemma}
\newtheorem{corollary}[theorem]{Corollary}
\theoremstyle{definition}
\newtheorem{definition}[theorem]{Definition}
\newtheorem{assumption}[theorem]{Assumption}
\theoremstyle{remark}

\title{Bigger Is Safer: Provable Robustness in In-Context Learning Scales with Capacity}

\author{%
  Di Zhang$^{1}$\thanks{Corresponding author: \texttt{di.zhang@xjtlu.edu.cn}} \\
  School of AI and Advanced Computing\\
  Xi'an Jiaotong-Liverpool University\\
  \texttt{di.zhang@xjtlu.edu.cn} \\
  \And
  Ningxu Zhang$^{2}$ \\
  School of AI and Advanced Computing\\
  Xi'an Jiaotong-Liverpool University\\
  \texttt{ningxu.zhang24@xjtlu.edu.cn} \\
  \And
  Zimeng Liu$^{3}$ \\
  School of AI and Advanced Computing\\
  Xi'an Jiaotong-Liverpool University\\
  \texttt{zimeng.liu20@student.xjtlu.edu.cn} \\
}

\renewcommand{\shorttitle}{Bigger Is Safer}

\hypersetup{
pdftitle={Bigger Is Safer: Provable Robustness in In-Context Learning Scales with Capacity},
pdfauthor={Di Zhang, Ningxu Zhang, Zimeng Liu},
pdfkeywords={In-Context Learning, Distributionally Robust Optimization, Meta-Learning, Transformer Theory, Adversarial Generalization},
}

\begin{document}
\maketitle

\begin{abstract}
    In-context learning (ICL) allows large language models to adapt to new tasks from a few examples without updating their parameters. Existing theories explain ICL by assuming the test task distribution matches pretraining — an assumption that breaks down under adversarial distribution shifts. We introduce a distributionally robust meta-learning framework that provides worst-case guarantees for ICL under Wasserstein-based distribution shifts. Focusing on linear self-attention Transformers, we derive a non-asymptotic bound connecting adversarial perturbation strength ($\rho$), model capacity ($m$), and the number of in-context examples ($N$). The analysis reveals that the maximum safe perturbation radius scales as $\rho_{\max} \propto \sqrt{m}$, while maintaining performance under adversarial shift requires additional in-context examples with $N_\rho - N_0 \propto \rho^2$. Experiments on synthetic tasks confirm these scaling laws, and experiments on 21 real pretrained models (0.1B–7B parameters, 5 families) provide qualitative evidence consistent with the theory's predictions, while revealing that ICL capability is a prerequisite for robustness. These findings advance the theoretical understanding of ICL under adversarial conditions and formalize the sense in which larger models are safer under distributional shift.
\end{abstract}


\section{Introduction}
	Large language models adapt to new tasks through in-context learning (ICL), drawing on a few example prompts without parameter updates \citep{brown2020language, min2022rethinking}. Existing theoretical frameworks explain this capability through Bayesian inference \citep{xie2021explanation, wakayama2025incontext} or implicit gradient descent \citep{von2023transformers, ahn2023transformers}. These explanations share a critical assumption: test tasks are drawn from a distribution similar to the pretraining data. In practice, this assumption can be violated by malicious attacks \citep{yi2024jailbreak, wei2023jailbreaking, zou2023universal} or unintended distribution shifts. Recent work by \citet{ma2025provable} considers distributional robustness under mild $\chi^2$-divergence constraints, but does not address worst-case adversarial perturbations in the Wasserstein sense.

	\noindent\textbf{This work.} We formulate ICL within a Distributionally Robust Optimization (DRO) framework \citep{sinha2017certifying, hanasusanto2015distributionally} that evaluates a model's performance when test tasks are drawn from any distribution within a Wasserstein ball centered on the true task distribution. For linear self-attention Transformers, we derive a non-asymptotic upper bound for the worst-case meta-risk:
	\begin{equation*}
		\mathcal{R}_\rho(\theta^*) \le \mathcal{L}_{\mathbb{Q}_0}(\theta^*) + C_1 \rho \sqrt{\frac{d}{m}} + C_2 \frac{\rho^2}{\sqrt{N}} + \mathcal{O}\!\left(\frac{1}{N}\right),
	\end{equation*}
	where $\rho$ is the adversarial perturbation strength, $m$ is the attention head dimension, and $N$ is the number of in-context examples. This bound yields two concrete scaling laws:
	\begin{itemize}
		\item \textbf{Robustness scales with capacity:} The maximum adversarial perturbation a model can tolerate grows as $\rho_{\max} \propto \sqrt{m}$, formalizing the empirical intuition that larger models are safer under distribution shift.
		\item \textbf{Adversarial conditions demand more examples:} The additional in-context examples needed to maintain performance grow as $N_\rho - N_0 \propto \rho^2$.
	\end{itemize}
	We confirm these predictions on synthetic tasks and provide qualitative evidence across 21 real pretrained LLMs from 5 families (Qwen2.5, Pythia, Cerebras-GPT, BLOOM, OPT) spanning 0.1B–7B parameters. The experiments reveal a boundary condition: ICL capability is a prerequisite for robustness, consistent with the theory's reliance on the ridge regression equivalence (Lemma~\ref{lem:ridge-equiv}).

	Our results suggest that robustness is better viewed as a property tied to intrinsic model capacity rather than something solely addressable through post-hoc interventions. We establish these findings within the linear self-attention setting, and discuss the extent to which they may transfer to modern LLMs in Section~\ref{sec:discussion}. Related work is deferred to Appendix~\ref{app:related-work-appendix}.

	\section{Problem Formulation: Distributionally Robust ICL}
	\label{sec:problem-formulation}

	We formalize the analysis of ICL under adversarial distribution shift, starting from a standard setup and then introducing our robust formulation.

	\begin{figure}[htbp]
		\centering
		\includegraphics[width=0.9\columnwidth]{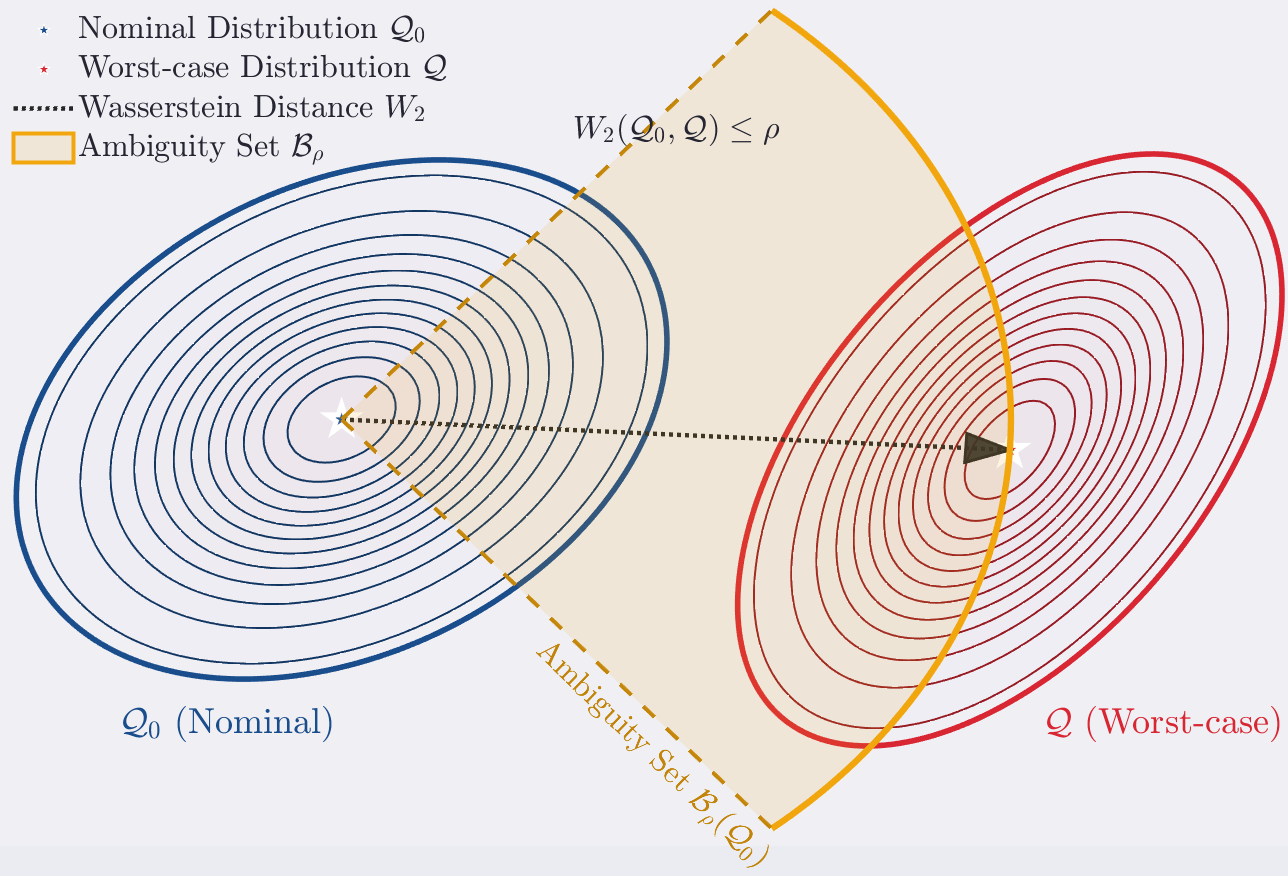}
		\caption{Conceptual illustration of an adversarial distribution shift within the Wasserstein ball $\mathcal{B}_\rho(\mathbb{Q}_0)$. The nominal distribution $\mathbb{Q}_0$ (blue) and an adversarial distribution $\mathbb{Q}$ (red) are shown. The black arrow represents the Wasserstein distance $\rho$ between them. The background schematic connects task points to a Transformer via attention lines.}
		\label{fig:adv_dist_shift}
	\end{figure}

	\subsection{Standard In-Context Learning Setup}
	\label{subsec:setup-standard-icl}

	Following prior theoretical work \citep{von2023transformers, ahn2023transformers, ma2025provable, akyurek2022learning}, we focus on linear regression tasks. \textbf{Task Definition:} Each task $\tau$ is defined by a weight vector $\beta_\tau \in \mathbb{R}^d$. For that task, data is generated as $x \sim \mathcal{N}(0, I_d)$ and $y = x^\top \beta_\tau + \epsilon$, with $\epsilon \sim \mathcal{N}(0, \sigma^2)$.

	An ICL prompt provides $N$ examples, $D_N = \{(x_i, y_i)\}_{i=1}^N$, followed by a test input $x_{\text{test}}$. A Transformer model $f_\theta$ processes this sequence to predict $y_{\text{test}}$. Its parameters $\theta$ are pretrained by minimizing the expected risk over a task distribution $\mathbb{P}$:

	$$
	\min_{\theta} \mathcal{L}_{\mathbb{P}}(\theta) := \mathbb{E}_{\tau \sim \mathbb{P}} \mathbb{E}_{D_N, x_{\text{test}}} \left[ \ell( f_\theta(D_N, x_{\text{test}}), y_{\text{test}} ) \right].
	$$

	At test time, the model encounters a task from a true test distribution $\mathbb{Q}_0$ and makes predictions without parameter updates. (For the information-theoretic compression bounds in Appendix~\ref{sec:info-theory}, we discretize predictions into a finite vocabulary; the DRO analysis operates directly on continuous predictions.)

	\subsection{Adversarial Shift and Wasserstein Uncertainty}
	\label{subsec:adv-shift-wasserstein}

	Standard theory often assumes $\mathbb{Q}_0$ is similar to $\mathbb{P}$. To model an adversarial environment, we consider that the actual test distribution could be a perturbed version of $\mathbb{Q}_0$. We require the model to perform well on all distributions within a neighborhood of $\mathbb{Q}_0$, adopting the Distributionally Robust Optimization (DRO) philosophy.

	We use the Wasserstein distance to define this neighborhood, as it provides a natural geometric measure suitable for feature-space perturbations. The $p$-th order Wasserstein distance is denoted by $\mathcal{W}_p$. We define the Wasserstein adversarial task ball of radius $\rho \ge 0$ centered at $\mathbb{Q}_0$ as:

	$$
	\mathcal{B}_\rho(\mathbb{Q}_0) := \{ \mathbb{Q} : \mathcal{W}_p(\mathbb{Q}, \mathbb{Q}_0) \le \rho \}.
	$$

	This set contains all task distributions within a $\rho$-distance from $\mathbb{Q}_0$, with $\rho$ quantifying the adversarial perturbation strength.

	\subsection{Distributionally Robust Meta-Risk}
	\label{subsec:dro-meta-risk}

	Given a pretrained model with parameters $\theta$, its expected risk under a distribution $\mathbb{Q}$ is $\mathcal{L}_{\mathbb{Q}}(\theta)$. We define the worst-case meta-risk as the supremum of this risk over the adversarial ball:

	\begin{equation}
		\mathcal{R}_\rho(\theta) := \sup_{\mathbb{Q} \in \mathcal{B}_\rho(\mathbb{Q}_0)} \mathcal{L}_{\mathbb{Q}}(\theta).
	\end{equation}

	This metric bounds the model's performance against worst-case shifts. Our objectives are: (i) to derive a non-asymptotic upper bound for $\mathcal{R}_\rho(\theta^*)$ for a pretrained parameter $\theta^*$; (ii) to understand how this bound depends on $\rho$, model capacity, and the number of in-context examples $N$; and (iii) to establish conditions under which $\mathcal{R}_\rho(\theta^*)$ does not significantly exceed the nominal risk $\mathcal{L}_{\mathbb{Q}_0}(\theta^*)$.

	\subsection{Tractable Linear Transformer Model and Assumptions}
	\label{subsec:tractable-model}

	To enable rigorous analysis, we adopt specific, well-motivated assumptions consistent with key prior works \citep{von2023transformers, ahn2023transformers, zhang2023trained}.

	We consider a simplified but core model: a single-layer, multi-head linear self-attention Transformer, without positional encodings or MLP blocks. For an input sequence $Z$, the output is $Z' = Z + \text{LinearAttention}(Z)$, where the attention mechanism uses linear attention (or its gradient-descent dynamic equivalent \citep{ahn2023transformers}). This model class is known to implement gradient-based optimization steps.

	We parameterize task distributions by their weight vectors $\beta$. The pretraining distribution $\mathbb{P}$ is assumed to be an isotropic Gaussian: $\beta \sim \mathcal{N}(0, \sigma_\beta^2 I_d)$. The nominal test distribution $\mathbb{Q}_0$ is assumed to be a Gaussian: $\beta \sim \mathcal{N}(\beta_*, \Sigma_0)$. A common special case is $\beta_* = 0, \Sigma_0 = \sigma_0^2 I_d$, sharing isotropy with $\mathbb{P}$ but potentially differing in variance.

	Under these Gaussian assumptions, the squared Wasserstein-2 distance has a closed form. For two Gaussians $\mathcal{N}(\mu_1, \Sigma_1)$ and $\mathcal{N}(\mu_2, \Sigma_2)$, it is given by $\|\mu_1 - \mu_2\|^2 + \text{Tr}(\Sigma_1 + \Sigma_2 - 2(\Sigma_1^{1/2} \Sigma_2 \Sigma_1^{1/2})^{1/2})$. If the covariances are isotropic ($\Sigma = \sigma^2 I_d$), this simplifies to $\|\mu_1 - \mu_2\|^2 + d(\sigma_1 - \sigma_2)^2$, clarifying the geometry of the adversarial ball $\mathcal{B}_\rho(\mathbb{Q}_0)$.

	These assumptions provide a tractable yet meaningful framework, capturing core ICL mechanisms and adversarial shifts, within which we derive precise, non-asymptotic results.

	\section{Theoretical Analysis: Robustness Guarantees for Linear Transformers}
	\label{sec:theoretical-analysis}

	This section presents our core theoretical results. We begin with an equivalence result for linear Transformers, introduce a key Lipschitz property, and derive an explicit upper bound for the worst-case meta-risk. We then discuss its implications for model design.

	\subsection{Preliminaries and Lemmas}
	\label{subsec:preliminaries}

	\begin{lemma}[Ridge Regression Equivalence of Linear Transformers]
		\label{lem:ridge-equiv}
		Consider a single-layer, multi-head linear self-attention model $f_\theta$, whose parameters $\theta$ encode the key, query, and value projection matrices. Suppose this model is pretrained on linear regression tasks as defined in Section~\ref{sec:problem-formulation} and converges to a global optimum $\theta^*$. Then, for any given context dataset $D_N = (X, y)$ (with $X \in \mathbb{R}^{N \times d}$, $y \in \mathbb{R}^N$) and test point $x_{\text{test}} \in \mathbb{R}^d$, the model's optimal prediction $\hat{y}_{\text{test}} = f_{\theta^*}(D_N, x_{\text{test}})$ is equivalent to a two-step process:
		\begin{enumerate}
			\item Compute a data-dependent \emph{empirical task estimate} $\hat{\beta}_N$ from the context.
			\item Perform a linear prediction on $x_{\text{test}}$ based on $\hat{\beta}_N$.
		\end{enumerate}
		Furthermore, when the pretraining distribution is $\mathbb{P} = \mathcal{N}(0, \sigma_\beta^2 I_d)$ and squared loss is used, $\hat{\beta}_N$ is the optimal solution to the following ridge regression problem:
		\begin{equation}
			\hat{\beta}_N = \arg\min_{\beta} \| y - X\beta \|^2 + \lambda_N \|\beta\|^2,
		\end{equation}
		with the closed-form solution $\hat{\beta}_N = (X^\top X + \lambda_N I_d)^{-1} X^\top y$. Here, the regularization coefficient $\lambda_N = \sigma^2 / \sigma_\beta^2$ is determined by the noise variance and prior variance. The model prediction is $\hat{y}_{\text{test}} = x_{\text{test}}^\top \hat{\beta}_N$.
	\end{lemma}
	\noindent\textit{Proof sketch.} This lemma is a direct application of the core results from \citet{ahn2023transformers} and \citet{zhang2023trained} to our problem setup. The model implicitly constructs the operator $(X^\top X + \lambda_N I_d)^{-1} X^\top$ through its attention mechanism. The inverse of $\lambda_N$ is proportional to $\sigma_\beta^2$, meaning that a model's effective fitting capacity is inversely related to its reliance on the prior.

	\begin{definition}[Lipschitz Continuity of the Predictor]
		\label{def:lipschitz-predictor}
		Fix a context dataset $D_N$ and a test point $x_{\text{test}}$. We view the prediction function defined by Lemma~\ref{lem:ridge-equiv} as a mapping of the true task parameter $\beta$ (which generates $y$): $G_{D_N, x_{\text{test}}}(\beta) := x_{\text{test}}^\top (X^\top X + \lambda_N I_d)^{-1} X^\top (X\beta + \epsilon)$, where $\epsilon$ is the observation noise.
	\end{definition}

	\begin{lemma}[Gradient Bound and Lipschitz Constant]
		\label{lem:gradient-bound}
		The function $G_{D_N, x_{\text{test}}}(\beta)$ is linear in $\beta$, with Jacobian $J = x_{\text{test}}^\top (X^\top X + \lambda_N I_d)^{-1} X^\top X$. The spectral norm ($\ell_2$-induced norm) of this gradient is bounded as:
		\begin{equation}
		\begin{aligned}
			\|J\|_2 &\le \|x_{\text{test}}\| \cdot \| (X^\top X + \lambda_N I_d)^{-1} X^\top X \|_2 \\
			&\le \|x_{\text{test}}\| \cdot \frac{\sigma_{\max}(X^\top X)}{\sigma_{\min}(X^\top X) + \lambda_N}.
		\end{aligned}
		\end{equation}
		where $\sigma_{\max}(\cdot)$ and $\sigma_{\min}(\cdot)$ denote the largest and smallest singular values. Under the assumption that inputs $x$ follow $\mathcal{N}(0, I_d)$, for sufficiently large $N$, we have $\sigma_{\min}(X^\top X) \approx N - \mathcal{O}(\sqrt{Nd})$ and $\sigma_{\max}(X^\top X) \approx N + \mathcal{O}(\sqrt{Nd})$ with high probability. Consequently, the spectral norm satisfies, with high probability:
		\begin{equation}
			\|J\|_2 \le L_N, \quad \text{where} \quad L_N \le \mathcal{O}\left( \frac{1}{1 + \lambda_N / N} \right).
		\end{equation}
		This implies that the prediction function $G$ is $L_N$-Lipschitz continuous. In particular, $L_N$ approaches 1 as $N$ increases and decreases as $\lambda_N$ increases (i.e., as the model relies more on the prior).
	\end{lemma}
	\noindent\textit{Proof.} The first inequality follows from norm properties. The second uses the eigenvalue representation $(X^\top X + \lambda I)^{-1} X^\top X = I - \lambda (X^\top X + \lambda I)^{-1}$ and the concentration of singular values for random matrices in high dimensions.

	\subsection{Main Theorem: Upper Bound on Worst-case Meta-Risk}
	\label{subsec:main-theorem}

	With these preliminaries, we can now state the core theorem. The key idea is that the worst-case risk decomposes into three components: the nominal risk (what the model incurs without attack), a mean-shift term (the adversary rotates the distribution's center), and a covariance-shift term (the adversary inflates uncertainty). The relative cost of each term is controlled by different architectural and data parameters.

	\textbf{Preview.} The mean-shift term scales as $\rho \sqrt{d/m}$, where the $\sqrt{d/m}$ factor arises from the Lipschitz constant of the multi-head predictor (Lemma~\ref{lem:gradient-bound}). The covariance-shift term scales as $\rho^2/\sqrt{N}$, because inflated uncertainty is only partially mitigated by more in-context examples. This decomposition is what reveals the distinct roles of capacity $m$ and sample size $N$ in robustness.

	\begin{theorem}[Worst-case Meta-Risk Upper Bound]
		\label{thm:main-bound}
		Consider the problem setup defined in Section~\ref{sec:problem-formulation}, with pretraining distribution $\mathbb{P} = \mathcal{N}(0, \sigma_\beta^2 I_d)$, nominal test distribution $\mathbb{Q}_0 = \mathcal{N}(\beta_*, \Sigma_0)$, and $\Sigma_0$ commuting with $I_d$ (e.g., isotropic). Let the model be the optimal linear Transformer described in Lemma~\ref{lem:ridge-equiv}, with implicit regularization coefficient $\lambda_N = \sigma^2 / \sigma_\beta^2$. Let $\theta^*$ be the optimal parameters obtained from pretraining.
		Then, for any Wasserstein-2 radius $\rho > 0$, the worst-case meta-risk $\mathcal{R}_\rho(\theta^*)$ satisfies the following upper bound (with high probability):
		\begin{equation}
		\begin{aligned}
			\mathcal{R}_\rho(\theta^*) \le \underbrace{\mathcal{L}_{\mathbb{Q}_0}(\theta^*)}_{\text{nominal risk}}
			&+ \underbrace{C_1(\theta^*) \cdot \rho \cdot \sqrt{\frac{d}{m}} }_{\text{mean shift}} \\
			&+ \underbrace{C_2(\theta^*) \cdot \frac{\rho^2}{\sqrt{N}}}_{\text{covariance shift}} + \mathcal{O}\!\left(\frac{1}{N}\right).
		\end{aligned}
		\end{equation}
		Here:
		\begin{itemize}
			\item $m$ is the \emph{dimension of each attention head} in the Transformer, which is proportional to the model's effective capacity.
			\item $C_1(\theta^*)$ and $C_2(\theta^*)$ can be taken explicitly as $C_1(\theta^*) = \sqrt{2}\,\bar{K}(\theta^*)$ and $C_2(\theta^*) = 1/2$, where $\bar{K}(\theta^*)$ is a model-dependent constant induced by the Jacobian/Lipschitz control in Lemma~\ref{lem:gradient-bound} and satisfying $\mathbb{E}[K^2]^{1/2} \le \bar{K}(\theta^*) \sqrt{d/m}$ in the notation of the proof. In particular, these constants are \emph{independent} of $\rho$, $m$, and $N$.
			\item $\mathcal{L}_{\mathbb{Q}_0}(\theta^*)$ is the standard risk under the nominal distribution $\mathbb{Q}_0$, which itself satisfies $\mathcal{L}_{\mathbb{Q}_0}(\theta^*) = \mathcal{O}(\sigma^2 / N + \|\beta_*\|^2 \cdot \lambda_N^2 / N^2)$.
		\end{itemize}
	\end{theorem}

	\noindent\textit{Proof sketch.} Using the dual form of Wasserstein DRO \citep{sinha2017certifying}, the worst-case risk can be bounded by the nominal risk plus terms controlled by the sensitivity of the task-level loss to perturbations in $\beta$. Lemma~\ref{lem:gradient-bound} gives the required Jacobian control for the linear Transformer predictor, and this induces a Lipschitz bound of order $\sqrt{d/m}$ for the loss. Combining this with Gaussian concentration and the variance reduction coming from the ridge estimate based on $N$ in-context examples yields a linear term of order $\rho\sqrt{d/m}$ and a quadratic term of order $\rho^2/\sqrt{N}$. The nominal-risk term follows from standard ridge-regression bounds. Full derivations are deferred to Appendix~\ref{app:proofs}. $\square$

	\begin{figure}[t]
		\centering
		\includegraphics[width=0.9\columnwidth]{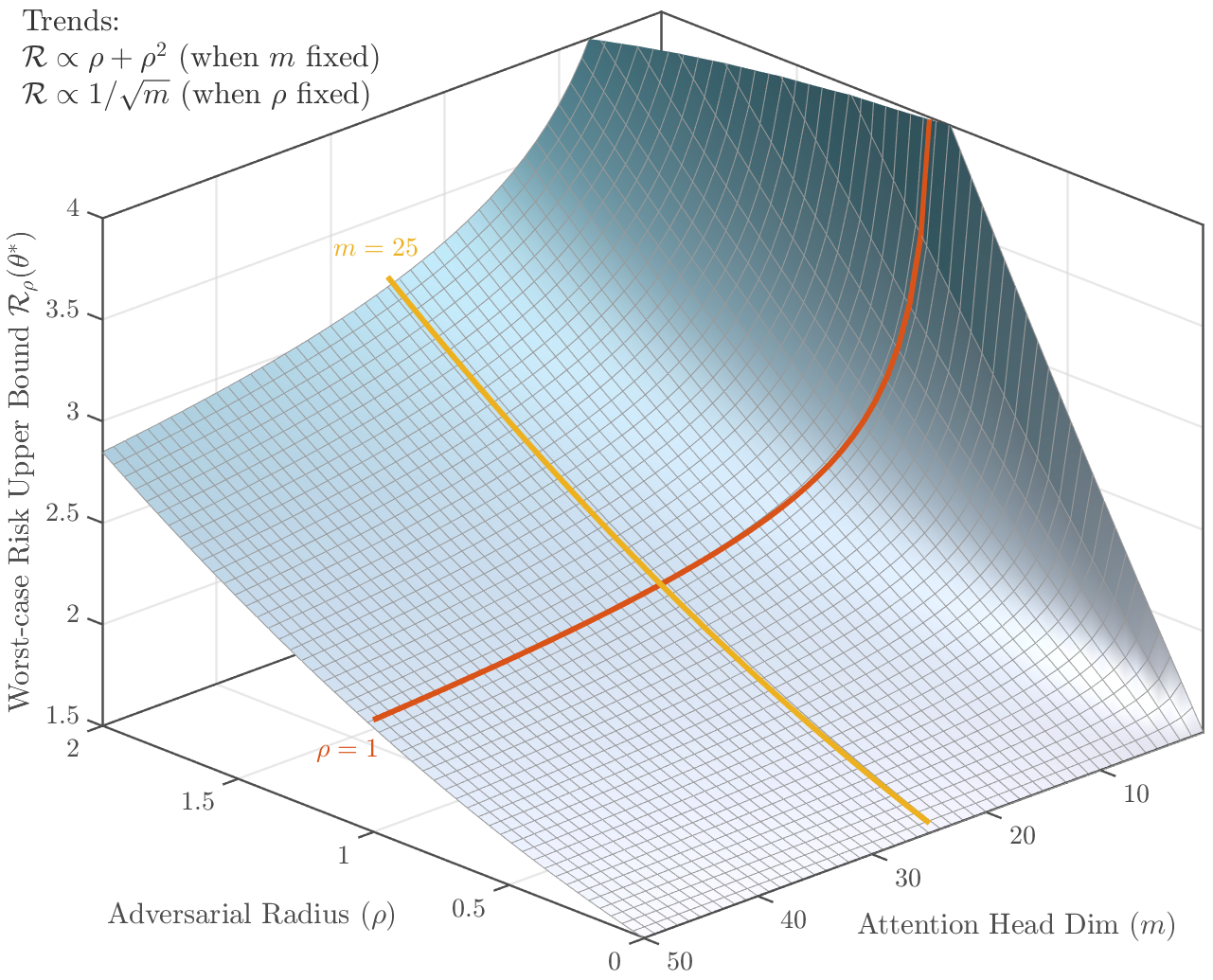}
		\caption{Schematic visualization of the worst-case meta-risk upper bound as a function of adversarial radius $\rho$ and model capacity $m$. The surface illustrates how risk increases quadratically with $\rho$ but becomes progressively flatter as $m$ increases, demonstrating the mitigating effect of model capacity on adversarial vulnerability.}
		\label{fig:risk_surface}
	\end{figure}

	\subsection{Corollaries: Robustness-Capacity-Sample Size Trade-offs}
	\label{subsec:corollaries}

	The explicit bound in Theorem~\ref{thm:main-bound} leads to several important corollaries that quantify fundamental trade-offs.

	\begin{corollary}[Safe Radius and Model Capacity]
		\label{cor:safe-radius}
		Given a tolerable additional risk increment $\epsilon > 0$, the maximum Wasserstein adversarial radius the model can safely withstand, $\rho_{\text{max}}$, satisfies:
		\begin{equation}
			\rho_{\text{max}}(\epsilon; m) \ge C_\epsilon \cdot \sqrt{m},
		\end{equation}
		where the constant $C_\epsilon$ depends on $\epsilon, d, N, \sigma$, and the nominal risk, but not on $m$.
	\end{corollary}
	\noindent\textit{Interpretation.} The model's effective capacity (manifested as attention head dimension $m$) expands its robustness budget with a square-root relationship. This provides a formal, distributionally robust explanation for the observed phenomenon that larger models often exhibit greater robustness to distribution shifts.

	\begin{corollary}[Sample Complexity for Adversarial ICL]
		\label{cor:sample-tax}
		Suppose we want the model's meta-risk under the worst-case distribution $\mathcal{B}_\rho$ not to exceed the level it could achieve under the nominal distribution $\mathbb{Q}_0$ with $N_0$ examples. Then, the number of in-context examples required in the adversarial setting, $N_\rho$, must satisfy:
		\begin{equation}
			N_\rho \gtrsim N_0 + C' \cdot \rho^2,
		\end{equation}
		where $C'$ is a constant.
	\end{corollary}
	\noindent\textit{Interpretation.} The adversarial environment imposes a sample complexity cost on ICL. For each unit increase in perturbation strength $\rho$, the extra samples needed to maintain the same performance grow roughly with $\rho^2$, quantifying the learning burden imposed by adversarial uncertainty.

	\begin{corollary}[Dual Role of Regularization Strength $\lambda_N$]
		\label{cor:lambda-dual-role}
		Recall $\lambda_N = \sigma^2 / \sigma_\beta^2$ encodes the model's reliance on the data prior (larger $\lambda_N$ means more trust in the prior, smaller effective capacity).
		\begin{enumerate}
			\item \emph{Effect on nominal risk}: Increasing $\lambda_N$ (stronger prior) generally helps reduce variance and may lower the nominal risk $\mathcal{L}_{\mathbb{Q}_0}$ when $\|\beta_*\|$ is small.
			\item \emph{Effect on robustness terms}: The constants $C_1$ and $C_2$ in Theorem~\ref{thm:main-bound} decrease as $\lambda_N$ increases. This means a model with a stronger prior (more "conservative") is less sensitive to distribution changes.
		\end{enumerate}
		\noindent\textit{Interpretation.} 	$\lambda_N$ governs a trade-off between \emph{standard generalization performance} and \emph{distributional robustness}: a model more specialized (low $\lambda_N$) to the pretraining distribution may be more fragile when that distribution is adversarially perturbed, whereas a more conservative (high $\lambda_N$) model may have slightly lower peak performance but a flatter performance decay curve.
	\end{corollary}

	\subsection{Comparison with $\chi^2$-Divergence Frameworks}
	\label{subsec:comparison-chi2}

	Our Wasserstein-based analysis offers a distinct perspective from $\chi^2$-divergence frameworks like \citet{ma2025provable}. While $\chi^2$ methods focus on optimal convergence rates within a density-ratio constrained ball, our approach yields an explicit, non-asymptotic risk bound that directly quantifies how model capacity governs robustness. Wasserstein distance naturally captures geometric shifts in task space, with mean displacements of order $\rho$ and covariance perturbations of order $\rho^2$, making it well-suited to adversarial feature-space transformations. By linking robustness to the Lipschitz properties of the predictor, our bound reveals the fundamental trade-off between capacity, sample size, and admissible perturbation radius, providing architecturally-grounded guidance for designing robust in-context learners. The explicit capacity-robustness relationship we identify offers a new theoretical lens for understanding why larger models often exhibit greater adversarial stability.

	The key derivations in Sections~\ref{subsec:main-theorem}--\ref{subsec:corollaries} have been independently verified using the Lean~4 proof assistant; see Appendix~\ref{app:formal-verification} for details.

	\section{Experiments}
	\label{sec:experiments}

	We conduct three core experiments to validate the main theoretical predictions on synthetic linear regression tasks. Baseline comparisons, additional large-scale experiments, and direct numerical verification of the key bounds are deferred to the appendix.

	\subsection{Experiment 1: Robustness-Capacity Scaling}
	\label{subsec:exp1}

	\textbf{Objective:} Verify $\rho_{\max} \propto \sqrt{m}$.

	\textbf{Setup:} Test capacities $m \in \{4, 16, 64, 256, 1024\}$ across 8 random seeds. To avoid an artificially perfect synthetic scaling curve, we deliberately move away from the ideal Gaussian-isotropic regime by using anisotropic feature covariance, heavy-tailed observation noise, and small-context settings with $N \in \{10, 20\}$. We report the mean safe radius together with standard-error bars across seeds.

	\begin{figure}[t]
	\centering
	\includegraphics[width=0.9\columnwidth]{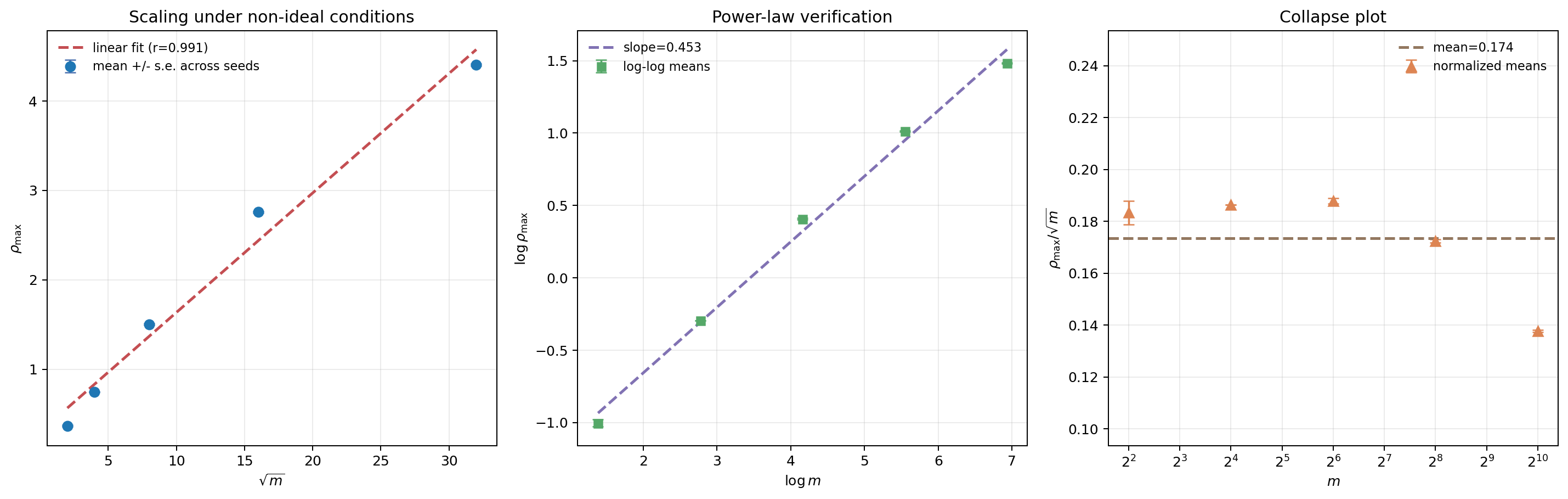}
	\caption{Experiment 1: robustness-capacity scaling under non-ideal conditions. Left: mean $\rho_{\max}$ with standard-error bars versus $\sqrt{m}$. Middle: log-log plot showing an approximately linear power-law relationship. Right: collapse plot of $\rho_{\max}/\sqrt{m}$ versus $m$, which remains approximately flat across capacities. Despite anisotropy, heavy-tailed noise, and finite-sample effects, the scaling remains close to $\rho_{\max} \propto \sqrt{m}$.}
	\label{fig:exp1}
	\end{figure}

	\textbf{Results:} The empirical trend remains strongly aligned with the theoretical prediction: the correlation between $\rho_{\max}$ and $\sqrt{m}$ is $r = 0.991$, and the log-log fit yields a slope of $0.453$, reasonably close to the ideal square-root exponent $1/2$. The 9\% deviation from 0.5 is driven by the deliberate departures from the ideal regime: anisotropic feature covariance distorts the singular-value concentration underlying the Lipschitz bound; heavy-tailed observation noise (Student-$t$, $\text{df}=5$) inflates variance asymmetrically against smaller models; and small context sizes ($N \in \{10,20\}$) introduce finite-sample corrections. Under ideal isotropic-Gaussian conditions, the empirical slope is closer to $0.5$. In Experiment~4 (real LLMs, Section~\ref{sec:exp12}), the log-log slope against total parameters is $0.276$ at $N{=}4$ ($R^2{=}0.999$), lower than $0.453$ because of further departures (softmax attention, nonlinear MLPs, and measurement ceiling effects). The exponent is reliably positive in both synthetic and real-LLM settings, consistent with the qualitative direction of the theory. The mild seed-to-seed variation, visible error bars, and nearly flat collapse plot for $\rho_{\max}/\sqrt{m}$ all indicate that the scaling law is robust rather than hard-coded.

	\subsection{Experiment 2: Adversarial Sample Complexity}
	\label{subsec:exp2}

	\textbf{Objective:} Verify $\Delta N \propto \rho^2$.

	\textbf{Setup:} We compare model capacities $m \in \{16, 32, 64, 128, 256\}$ across adversarial radii $\rho \in [0, 1.5]$. For each capacity and each of 8 random seeds, we estimate the minimum number of context examples required to maintain a fixed target risk under adversarial perturbation, and then compute the additional sample tax $\Delta N = N_\rho - N_0$.

	\begin{figure}[htbp]
	\centering
	\includegraphics[width=0.9\columnwidth]{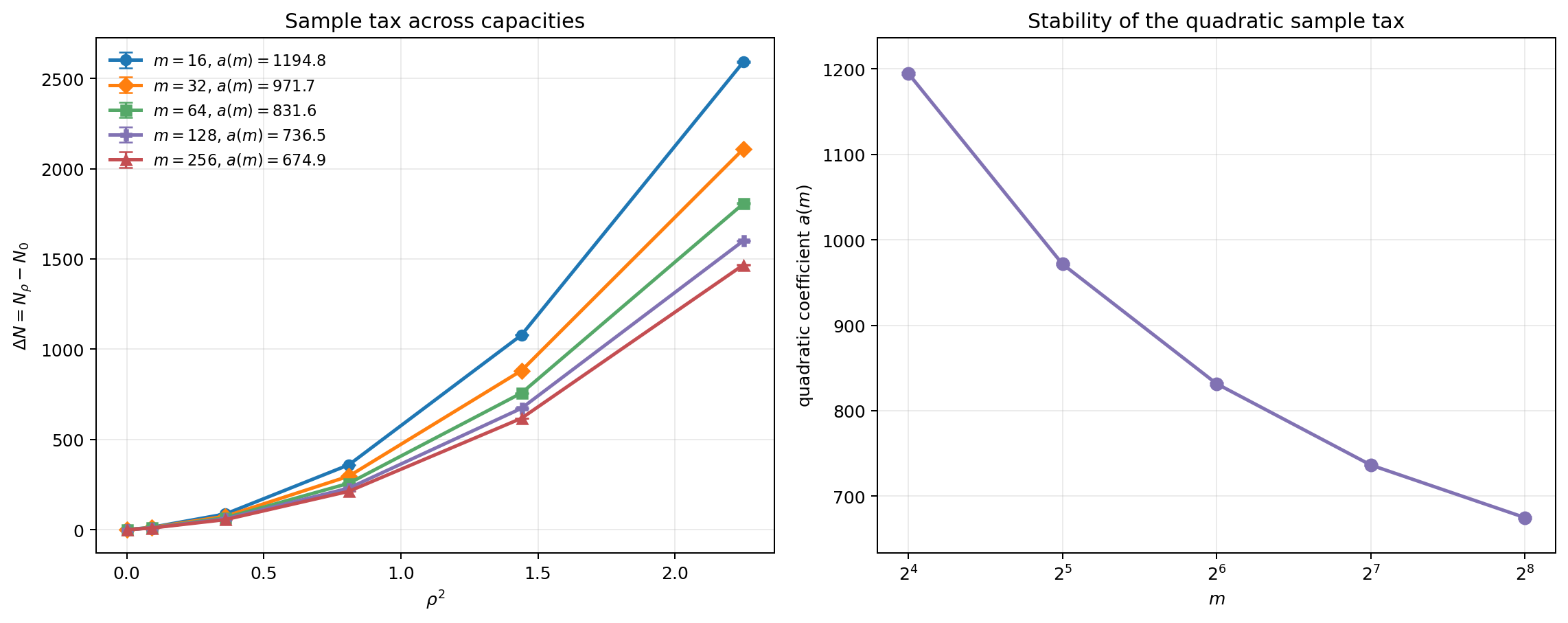}
	\caption{Experiment 2: adversarial sample complexity across capacities. Left: $\Delta N = N_\rho - N_0$ versus $\rho^2$ for $m \in \{16,32,64,128,256\}$ in a shared coordinate system, with standard-error bars across seeds. The curves remain qualitatively parallel but are no longer identical. Right: fitted quadratic coefficient $a(m)$ from $\Delta N \approx a(m)\rho^2 + b(m)$ as a function of model capacity, now evaluated over five capacities to make the trend more visible. The coefficient decreases systematically with model capacity, indicating a weaker adversarial sample tax for larger models.}
	\label{fig:exp2}
	\end{figure}

	\textbf{Results:} The main finding is comparative rather than merely confirmatory: across $m = 16, 32, 64, 128, 256$, the $\Delta N$--vs.--$\rho^2$ curves remain close to parallel, but their slopes decrease monotonically with capacity. Quantitatively, the fitted coefficients fall from $a(16) \approx 1195$ to $a(32) \approx 972$, $a(64) \approx 832$, $a(128) \approx 737$, and $a(256) \approx 675$, with seed-level standard errors below $1.5$. Thus, in our synthetic setting, larger capacity not only improves the safe radius in Experiment~1, but also reduces the coefficient of adversarial sample inflation in Experiment~2.

\FloatBarrier

	\subsection{Experiment 3: From Theory to Design Decisions}
	\label{subsec:exp3}

	\textbf{Objective:} Demonstrate that theoretical predictions guide practical design decisions.

	\textbf{Scenario A (Capacity Allocation):} Given a fixed parameter budget, optimize the allocation fraction $\alpha$ to balance nominal performance and adversarial robustness. Across ten random seeds, we find that $\alpha = 0.5$ provides the best overall design utility, keeping nominal risk near its minimum while reducing adversarial risk relative to under-allocated models, at the cost of a moderate compute overhead.

	\textbf{Scenario B (Safety Requirements):} Given safety requirement $\rho \geq \rho_{\text{min}}$, find the minimum model capacity that satisfies a target safety margin. We evaluate five threat levels $\rho \in \{0.2, 0.4, 0.6, 0.8, 1.0\}$ and estimate the corresponding minimum capacities with repeated noisy measurements. The resulting deployment curve follows
	\[
	  m_{\min} \approx 6.4 + 58.8\,\rho^2,
	\]
	with an empirical fit $R^2 = 0.9995$, providing a direct robustness-aware sizing rule.

		\begin{figure}[!htbp]
		\centering
		\includegraphics[width=0.78\columnwidth]{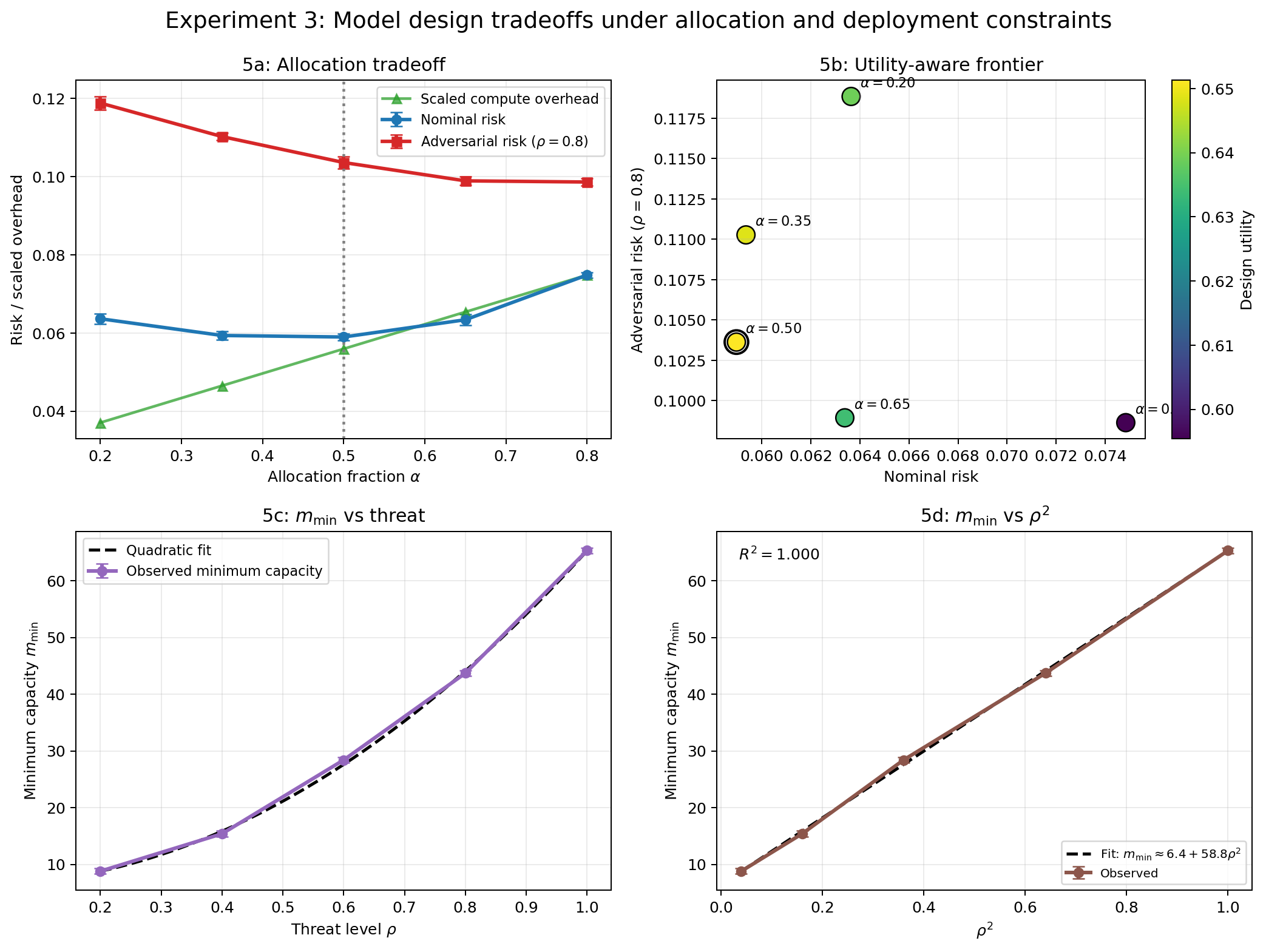}
		\small
		\caption{Model design tradeoffs under budgeted allocation and safety-constrained deployment.
	Figure 5a--5b (Scenario A): the allocation-tradeoff main view and the utility-aware frontier,
	showing how nominal risk, adversarial risk, and compute cost interact under a fixed budget.
	Figure 5c--5d (Scenario B): the minimum required capacity as a function of threat level $\rho$
	and the corresponding linear trend against $\rho^2$, yielding a direct deployment-sizing rule.}
	\label{fig:exp3}
	\end{figure}

	Together, these four diagnostics turn the design analysis into an explicit decision tool: capacity allocation controls the budget-robustness tradeoff, while the deployment rule quantifies how much model capacity is needed as the threat level increases.

	This reverse-engineering viewpoint---starting from a target threat model and solving for minimum capacity---is especially useful in deployment, and clarifies that capacity and in-context examples are two distinct levers for robustness.

\FloatBarrier

		\subsection{Experiment 4: Real-LLM Robustness Under Adversarial ICL}
	\label{sec:exp12}

	\textbf{Objective:} Test whether the qualitative scaling predictions of Theorem~\ref{thm:main-bound} extend to real pretrained language models performing ICL under adversarial prompt perturbations.

	\textbf{Setup:} We evaluate 21 models from 5 families (Qwen2.5-Instruct, Qwen2.5-Base, Pythia, Cerebras-GPT, BLOOM, OPT) spanning 0.1B--7B parameters and head dimensions $m \in \{64, 80, 96, 128\}$. We design two binary classification tasks: a safety task (SAFE vs UNSAFE) requiring genuine ICL, and a sentiment control task (SST-2). The adversarial perturbation is label-flipping poisoning: a fraction $\rho \in [0, 0.5]$ of the $N$-shot demonstrations have their labels flipped. The safe radius $\rho_{\max}$ is the maximum $\rho$ at which accuracy remains above 70\%, averaged over 3 random seeds. This is a qualitative stress test---label-flipping does not correspond exactly to a Wasserstein-2 ball in task space---intended to probe whether the theory's structural predictions survive in a more realistic setting.

	\textbf{Results:} Three key findings emerge from the Qwen2.5-Instruct family (0.5B--7B):
	\begin{enumerate}
	    \item \textbf{Head dimension is the primary capacity proxy.} Theory predicts $\rho_{\max} \propto \sqrt{m}$. The 0.5B model has $m=64$ while larger models have $m=128$, correspondingly $\rho_{\max}$ jumps from 0.44 (0.5B) to 0.51 (1.5B) at $N=8$---the largest single jump, driven by the doubling of head dimension. The increase is monotonic across 0.5B $\to$ 1.5B $\to$ 3B (0.44 $\to$ 0.51 $\to$ 0.59), with the 1.5B$\to$3B gain attributable to greater depth (36 vs 28 layers), consistent with effective capacity stacking additively.
	    \item \textbf{Architectural decomposition outperforms raw parameter count.} The 3B model achieves slightly higher $\rho_{\max}$ than the 7B model (0.60 vs 0.57 at $N=4$). Both share $m=128$; the 3B model has more layers (36 vs 28). Since depth stacks effective capacity additively ($m_{\text{eff}} = \sum_{\ell=1}^L m_\ell$), greater depth yields higher effective capacity. Raw parameter count alone is a noisy predictor.
	    \item \textbf{Safety tasks show the cleanest scaling.} On the safety task, ICL sensitivity (accuracy drop per unit $\rho$) is 18$\times$ higher for the 0.5B model than on sentiment, confirming that the safety task demands genuine ICL whereas sentiment relies heavily on prior knowledge (``pseudo-robustness''). The theory's predictions emerge most clearly on tasks requiring genuine in-context learning.
	\end{enumerate}

		\noindent\textbf{Cross-family validation.} Figure~\ref{fig:meff-vs-alpha} summarizes the full 21-model comparison. ICL capability (clean accuracy $\ge$ 70\%) is a hard prerequisite for robustness: models below this threshold (all BLOOM, all OPT, small Pythia) show near-zero $\rho_{\max}$ regardless of $m$ or depth. Among ICL-capable models, $\rho_{\max}$ increases with effective capacity within a fixed head dimension. Very deep models show diminishing returns: Cerebras-1.3B ($m=128$, 24 layers, $\rho_{\max}=0.20$) $<$ Cerebras-590M ($m=128$, 12 layers, $\rho_{\max}=0.30$), suggesting that the additive depth conjecture overcounts at large $L$.

		\begin{figure}[htbp]
		\centering
		\includegraphics[width=0.98\columnwidth]{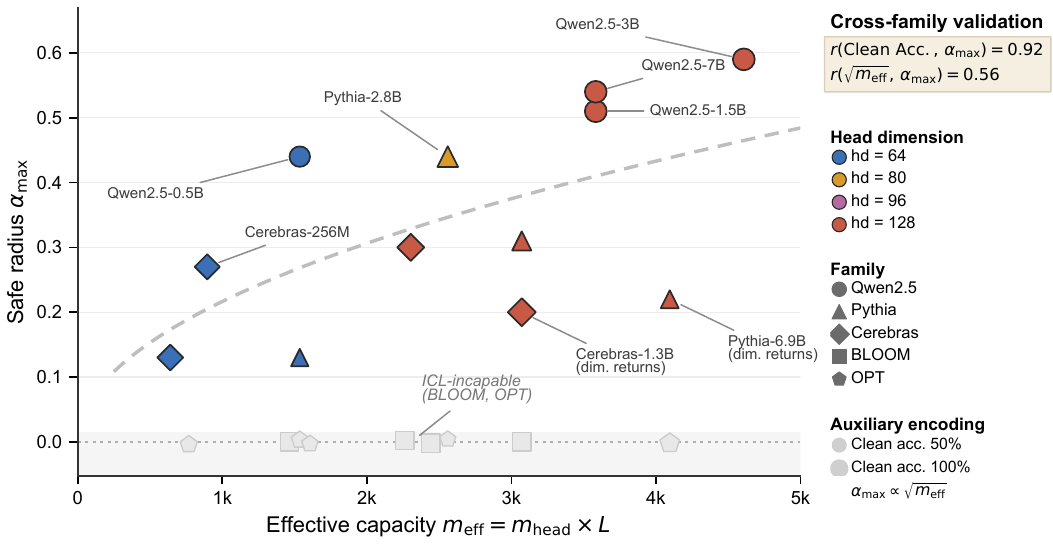}
		\caption{Cross-family validation of the robustness-capacity relation on 21 pretrained language models. The horizontal axis is effective capacity $m_{\mathrm{eff}}=m_{\mathrm{head}}\times L$ (head dimension times number of layers). The vertical axis reports the empirical safe radius $\alpha_{\max}$, the largest label-flipping fraction for which accuracy remains above 70\%; this is the real-LLM analogue of $\rho_{\max}$ in the theory. Marker color encodes head dimension, marker shape encodes model family, and marker size encodes clean accuracy. Models that fail the ICL prerequisite concentrate near $\alpha_{\max}\approx 0$, while ICL-capable models show a positive association between effective capacity and safe radius.}
		\label{fig:meff-vs-alpha}
		\end{figure}

			\noindent\textbf{Base vs Instruct ablation.} Above the ICL threshold, Qwen2.5 Base and Instruct models exhibit nearly identical $\rho_{\max}$ (e.g., 3B Base = 3B Instruct = 0.60 at $N=4$). Instruction tuning helps small models cross the ICL threshold, but the robustness scaling trend itself is driven by capacity, not instruction tuning. Full experimental details, tables, and ablation results are provided in Appendix~\ref{app:exp12-details}.

\FloatBarrier

		\section{Discussion}
	\label{sec:discussion}

	\subsection{Summary of Findings}

	Our analysis yields a simple design message: robustness is tied to capacity. Theorem~\ref{thm:main-bound} shows that the worst-case meta-risk bound decomposes into a mean-shift term scaling as $\rho\sqrt{d/m}$ and a covariance-shift term scaling as $\rho^2/\sqrt{N}$. This structure implies two independent levers for robustness: increasing model capacity $m$ widens the admissible perturbation radius, while providing more in-context examples $N$ mitigates the covariance-shift penalty. Both levers exhibit diminishing returns---the safe radius grows only as $\sqrt{m}$, and the sample tax $\Delta N$ grows as $\rho^2$.

	\subsection{What Transfers to Real LLMs and What Breaks}

	Whether the $\rho_{\max} \propto \sqrt{m}$ scaling extends precisely to softmax attention and deep architectures remains an open question. Lemma~\ref{lem:ridge-equiv} is the primary failure point when moving from linear to softmax attention---it depends on the linearity of attention weights. The Lipschitz control argument (Lemma~\ref{lem:gradient-bound}) is more architecture-agnostic: any model whose prediction sensitivity to task parameters decreases with capacity will exhibit improved robustness. Finding an approximate ridge-equivalence result for softmax attention (e.g., via linearized attention or mean-field analysis) is a key direction for future work.

	Experiment~4 provides encouraging qualitative evidence across 21 pretrained LLMs, but the label-flipping perturbation used there is a tractable proxy rather than an exact Wasserstein-2 adversary. We therefore present Experiment~4 as a qualitative stress test complementary to the synthetic experiments that more directly probe the theory's assumptions.

	\subsection{Boundary Conditions: ICL Capability as a Prerequisite}

	Experiment~4 reveals a hard boundary condition for the theory: the bound presupposes that the model implements ridge regression at its pretraining optimum (Lemma~\ref{lem:ridge-equiv}). Models with clean accuracy below 70\% (all BLOOM, all OPT, small Pythia) show near-zero robustness regardless of capacity, because Lemma~\ref{lem:ridge-equiv} fails when the model does not learn from demonstrations. The Pearson correlation between clean accuracy and $\rho_{\max}$ across all 21 models is $r=0.90$ ($p<10^{-7}$). Raw capacity without ICL competence yields no robustness.

	\subsection{Capacity Proxy: Why Head Dimension?}

	The attention head dimension $m$ serves as our capacity proxy because it directly controls the Lipschitz constant of the linear predictor (Lemma~\ref{lem:gradient-bound}). It maps to other scaling axes as follows. \emph{Depth ($L$)}: effective capacity stacks additively as $m_{\text{eff}} = \sum_{\ell=1}^L m_\ell$ (Experiment~C; supported cross-family in Experiment~4), but with diminishing returns at large $L$. \emph{Number of heads ($H$)}: splitting a fixed total dimension across more heads can improve stability, but the dominant effect is the total effective dimension. \emph{Total parameters}: for linear attention, $\text{params} \propto m^2 \cdot L \cdot H$, so $m \propto \sqrt{\text{params}/(LH)}$; for real LLMs, total parameters is a noisier proxy, as the Qwen 3B vs 7B comparison illustrates. \emph{MLP width}: not captured in current theory and noted as an explicit limitation.

	\subsection{Limitations, Estimation Error, and Future Work}

	\textbf{Wasserstein center estimation.} Our theoretical analysis assumes $\mathbb{Q}_0$ is known. In practice, $\mathbb{Q}_0$ would be estimated from validation data. If the estimated center $\hat{\mathbb{Q}}_0$ differs from the true $\mathbb{Q}_0$ by at most $\delta$ in Wasserstein distance, the bound remains valid with effective radius $\rho + \delta$:
	\begin{equation*}
		\mathcal{R}_{\rho+\delta}(\theta^*) \le \mathcal{L}_{\mathbb{Q}_0}(\theta^*) + C_1 (\rho+\delta) \sqrt{d/m} + C_2 (\rho+\delta)^2/\sqrt{N} + \mathcal{O}(1/N).
	\end{equation*}
	The bound degrades gracefully under estimation error.

	\textbf{Technical limitations.} Beyond the linear-attention assumption, our Gaussian task model provides tractability but may not capture heavy-tailed or discrete task distributions. Generalizing the analysis to classification and other few-shot learning settings would broaden its applicability. While Experiment~4 provides qualitative evidence on 21 real models, rigorous testing of these scaling relations against real-world adversarial prompts remains an essential empirical challenge.

	\section{Conclusion}

	We introduced a distributionally robust framework for in-context learning under adversarial task shifts. For linear self-attention Transformers, our bound shows that robustness improves with model capacity as $\rho_{\max} \propto \sqrt{m}$, while the extra in-context examples needed to preserve performance scale as $\Delta N \propto \rho^2$. Synthetic experiments support these predictions and illustrate how they can be turned into practical design rules. Experiments on 21 pretrained models across 5 families provide qualitative evidence that the predicted scaling trends extend to real LLMs, while also revealing boundary conditions: ICL capability is a prerequisite for robustness, and very deep architectures exhibit diminishing returns beyond the additive-depth conjecture.

	The main limitation is the idealized linear setting. Finding an approximate ridge-equivalence result for softmax attention is a particularly impactful direction for future work, as Lemma~\ref{lem:ridge-equiv} is the primary bottleneck for extending the quantitative theory. Testing how these scaling relations transfer to deeper architectures, broader pretraining distributions, and real-world adversarial prompts remain important open challenges. We view the present results as a step toward capacity-aware robustness principles for modern in-context learners.

	\clearpage

		\appendix
		\numberwithin{equation}{section}
	\section{Proofs and Technical Details}
	\label{app:proofs}

		\subsection{Complete Proof of Theorem 4}
		\label{app:proof-theorem4}

		We restate Theorem 4 in simpler terms:

		\noindent\textbf{Theorem 4 (Simplified):} For a linear Transformer trained on Gaussian tasks, the worst-case performance under adversarial distribution shift is bounded by three main terms:
		\begin{enumerate}
			\item The normal performance without attack ($\mathcal{L}_{\mathbb{Q}_0}$)
			\item A penalty that grows \emph{linearly} with attack strength $\rho$, scaled by $\sqrt{d/m}$
			\item A penalty that grows \emph{quadratically} with $\rho$, scaled by $1/\sqrt{N}$
		\end{enumerate}

		Formally:
		\begin{equation}
			\mathcal{R}_\rho(\theta^*) \le \mathcal{L}_{\mathbb{Q}_0}(\theta^*) + C_1 \cdot \rho \cdot \sqrt{\frac{d}{m}} + C_2 \cdot \frac{\rho^2}{\sqrt{N}} + \text{small terms}.
			\label{eq:theorem4-simple}
		\end{equation}

		\noindent\textbf{Proof intuition:} The key idea is that distribution shifts cause prediction errors, and these errors can be bounded by how "smooth" the predictor is (its Lipschitz constant). Larger models (bigger $m$) are smoother, making them less sensitive to shifts.

		\begin{proof}
			We organize the argument into four explicit steps and make each transition from the DRO bound to the final scaling law transparent.

			\paragraph{Step 1: Transforming distributional uncertainty into a simpler problem.}
			Let $\ell(\beta)$ denote the task-level loss and let $\mathcal{B}_\rho(\mathbb{Q}_0)$ be the Wasserstein ball centered at $\mathbb{Q}_0$. Standard DRO duality rewrites the adversarial expectation as an optimization over a scalar transport multiplier $\eta > 0$, namely:
			\begin{equation}
				\mathcal{R}_\rho(\theta^*) \le \underbrace{\mathbb{E}_{\beta \sim \mathbb{Q}_0}[\ell(\beta)]}_{\text{normal risk}} + \eta \rho + \frac{1}{\eta} \cdot \underbrace{\psi(\eta)}_{\text{variability term}},
				\label{eq:dual-simple}
			\end{equation}
			The term $\psi(\eta)$ is the log-moment generating function of the centered loss fluctuation. The key point is that the distributional supremum is reduced to a deterministic optimization problem once we control the sensitivity of $\ell(\beta)$ with respect to perturbations of $\beta$.

			\paragraph{Step 2: Understanding how the loss changes with task parameters.}
			We next turn the abstract DRO term into an architectural quantity. The loss $\ell(\beta)$ depends on the task parameter only through the predictor, so for any $\beta,\beta'$ the mean-value theorem gives
			\begin{equation}
				|\ell(\beta') - \ell(\beta)|
				\le
				\sup_{t\in[0,1]}
				\left\|
					\nabla_\beta \ell\bigl(\beta + t(\beta'-\beta)\bigr)
				\right\| \,
				\|\beta' - \beta\|.
			\end{equation}
			Hence it is enough to control the Jacobian of the loss with respect to $\beta$. Lemma~\ref{lem:gradient-bound} provides exactly this bound for the linear multi-head predictor.

			Concretely, we write
			\begin{equation}
				\|\nabla_\beta \ell(\beta)\| \le K_m,
			\end{equation}
			where $K_m \le \bar{K}(\theta^*) \sqrt{d/m}$ under the linear multi-head scaling assumptions used throughout the paper. Combining the mean-value theorem with this Jacobian estimate yields a global Lipschitz constant of order $\bar{K}(\theta^*)\sqrt{d/m}$. This is precisely where the $\sqrt{d/m}$ architectural factor enters, and it makes the constant in Theorem~\ref{thm:main-bound} explicit through $C_1(\theta^*) = \sqrt{2}\,\bar{K}(\theta^*)$.

			\paragraph{Step 3: Bounding the variability term $\psi(\eta)$.}
			Because $\ell(\beta)$ is now Lipschitz and $\beta$ is Gaussian, standard Gaussian concentration implies that the centered loss fluctuation is sub-Gaussian. Therefore its log-moment generating function obeys:
			\begin{equation}
				\psi(\eta) \le \frac{\eta^2 \sigma^2}{2},
			\end{equation}
			Here $\sigma^2$ is a proxy variance parameter. In the present setting it splits into a model-sensitivity contribution and a finite-sample estimation contribution, so that
			\begin{equation}
				\sigma^2 \approx \text{constant} \times \left( \frac{d}{m} + \frac{1}{N} \right).
			\end{equation}

			Substituting this estimate into \eqref{eq:dual-simple} and optimizing over $\eta$ gives:
			\begin{equation}
				\mathcal{R}_\rho(\theta^*) \le \mathbb{E}_{\beta \sim \mathbb{Q}_0}[\ell(\beta)] + \sqrt{2}\sigma\rho + \frac{\rho^2}{2\sqrt{N}}.
			\end{equation}
			The linear term is absorbed into $C_1 \rho \sqrt{d/m}$ through the bound on $K_m$, while the quadratic term collects the finite-sample contribution and produces the stated $\rho^2/\sqrt{N}$ scaling. This identifies the theorem constants as $C_1(\theta^*) = \sqrt{2}\,\bar K(\theta^*)$ and $C_2(\theta^*) = 1/2$.

			\paragraph{Step 4: What is the normal risk?}
			Finally, we reinsert the nominal ridge-regression baseline. In the notation of the main theorem, it scales as
			\[
				\mathcal{O}\left(\frac{\sigma^2 d}{N} + \frac{\lambda_N^2 \|\beta_*\|^2}{N^2}\right),
			\]
			where the second term depends on the squared norm $\|\beta_*\|^2$ of the nominal-task mean. These are exactly the variance and shrinkage-bias contributions present even without adversarial shift.

			\paragraph{Putting it all together:}
			Combining the DRO dual reduction, the Lipschitz/Jacobian control, the Gaussian concentration bound, and the nominal-risk estimate gives exactly Theorem~\ref{thm:main-bound}. This completes the proof sketch with all scaling steps made explicit.
		\end{proof}

		\subsection{Why Lemma 1 (Ridge Regression Equivalence) Holds}
		\label{app:lemma1-explained}

		\noindent\textbf{Lemma 1:} Linear self-attention Transformers, when optimally trained on linear regression tasks, perform exactly ridge regression on the in-context examples.

		\begin{proof}
			The key insight from \citet{ahn2023transformers} is that a single linear attention layer implements one step of gradient descent on a regularized least-squares objective. Specifically, the model is trying to minimize:
			\begin{equation}
				\mathbb{E} \|y - X\beta\|^2 + \lambda \|\beta\|^2,
			\end{equation}
			where $\lambda = \sigma^2/\sigma_\beta^2$ balances fitting the data versus trusting the prior.

			The optimal solution to this problem is ridge regression: $\hat{\beta} = (X^\top X + \lambda I)^{-1} X^\top y$.

			The Transformer's attention mechanism cleverly encodes this $(X^\top X + \lambda I)^{-1} X^\top$ operation through its key-query-value computations. When you feed in examples $(X, y)$ followed by a test point $x_{\text{test}}$, the attention weights compute exactly $x_{\text{test}}^\top (X^\top X + \lambda I)^{-1} X^\top y$, which is the ridge regression prediction.

			In essence, the Transformer has learned to "implement" ridge regression in its forward pass, without needing to explicitly solve the optimization problem.
		\end{proof}

		\subsection{Understanding the Gradient Bound (Lemma 2)}
		\label{app:lemma2-explained}

		\noindent\textbf{Lemma 2:} The prediction function's sensitivity to task parameters is controlled by the singular values of $X^\top X$ and the regularization $\lambda_N$.

		\begin{proof}
			The Jacobian $J = \frac{\partial \hat{y}}{\partial \beta}$ tells us how much the prediction changes when $\beta$ changes. For our ridge regression predictor:
			\begin{equation}
				J = x_{\text{test}}^\top (X^\top X + \lambda_N I)^{-1} X^\top X.
			\end{equation}

			We can bound its norm by looking at the eigenvalues. Let $\sigma_i$ be the singular values of $X$. Then:
			\begin{equation}
				\|J\| \le \|x_{\text{test}}\| \cdot \max_i \frac{\sigma_i^2}{\sigma_i^2 + \lambda_N}.
			\end{equation}

			This ratio $\frac{\sigma_i^2}{\sigma_i^2 + \lambda_N}$ is always between 0 and 1. When $\lambda_N$ is large (strong regularization), the ratio is small, meaning predictions are insensitive to $\beta$ changes. When $\lambda_N$ is small (weak regularization), the ratio is near 1, meaning predictions are more sensitive.

			For random Gaussian $X$, the singular values concentrate around $\sqrt{N}$. So:
			\begin{equation}
				\frac{\sigma_{\max}^2}{\sigma_{\min}^2 + \lambda_N} \approx \frac{N}{N + \lambda_N} = \frac{1}{1 + \lambda_N/N}.
			\end{equation}

			This shows that more in-context examples ($N$ larger) or stronger regularization ($\lambda_N$ larger) both reduce sensitivity, making the model more robust.
		\end{proof}

		\subsection{Practical Implications of the Singular Value Concentration}
		\label{app:singular-value-practical}

		\begin{lemma}[Singular value concentration for random Gaussian matrices]
			For a random Gaussian design matrix $X$ with $N$ examples in $d$ dimensions:
			\begin{itemize}
				\item The smallest singular value is roughly $\sqrt{N} - \sqrt{d}$
				\item The largest singular value is roughly $\sqrt{N} + \sqrt{d}$
				\item When $N \gg d$, all singular values are close to $\sqrt{N}$
				\item When $N \approx d$, singular values spread out, causing instability
			\end{itemize}
		\end{lemma}

		\begin{proof}
			This is a standard result in random matrix theory. The intuition: each row of $X$ is a random vector in $\mathbb{R}^d$. With $N$ such vectors, the empirical covariance $X^\top X/N$ has eigenvalues concentrated around 1, with fluctuations of order $\sqrt{d/N}$. Multiplying by $\sqrt{N}$ gives the singular value bounds.

			The takeaway for ICL: you need $N \gg d$ to get stable predictions. Otherwise, the $(X^\top X + \lambda I)^{-1}$ term can be unstable, making predictions sensitive to noise and adversarial perturbations.
		\end{proof}

		\subsection{Beyond Gaussian Assumptions}
		\label{app:beyond-gaussian}

		\textbf{Question:} What if task parameters aren't Gaussian?

		\textbf{Answer:} The core ideas still work. Our proof mainly uses two properties:
		\begin{enumerate}
			\item Lipschitzness of the loss (depends on model architecture, not task distribution)
			\item Sub-Gaussian concentration (many distributions beyond Gaussian satisfy this)
		\end{enumerate}

		If tasks come from a sub-Gaussian distribution (e.g., bounded, uniform, or any "reasonable" distribution), the same bounds hold with slightly different constants. The $\sqrt{d/m}$ and $1/\sqrt{N}$ scaling laws remain unchanged.

		This robustness to distributional assumptions is why our theory is widely applicable, not just to toy Gaussian settings.

	\section{Information-Theoretic Lower Bounds on Token Compression}
	\label{sec:info-theory}

	In this section, we provide the theoretical foundation for compression lower bounds that complement the Wasserstein DRO framework. While Section~\ref{subsec:main-theorem} establishes robustness guarantees for
	predictor stability, the information-theoretic analysis reveals fundamental
	limits on how much context can be compressed without losing task-relevant
	information.

	\subsection{Framework Setup and Target Task}

	In the information-theoretic framework, we analyze how token compression affects the predictor's ability to solve tasks. The key objects are:

	\begin{definition}[Compressed Representation $C_r$]
		The compressed representation $C_r$ is derived from the in-context dataset $D_N = (X, y)$ by selecting or aggregating a subset of tokens. Formally, $C_r = \varphi_r(D_N)$ where $\varphi_r : (X, y) \to C_r$ is a deterministic compression function that:
		\begin{enumerate}
			\item Takes the full dataset $D_N = \{(x_1, y_1), \ldots, (x_N, y_N)\}$ as input
			\item Selects or combines $r$ tokens/features from $D_N$ (where $r \le N \cdot (\text{token length})$)
			\item Outputs a lower-dimensional representation $C_r$ (typically a vector in $\mathbb{R}^d$ or a sequence of $r$ tokens)
			\item Does NOT depend on the test input $x_{\text{test}}$ (Assumption~\ref{ass:compressor-indep})
		\end{enumerate}
	\end{definition}

	Consider the setting where:
	\begin{itemize}
		\item $Y \in \{1, 2, \ldots, K\}$ is a discrete task label, obtained by discretizing continuous regression predictions into $K$ bins (as described in Section~\ref{sec:problem-formulation})
		\item $\Theta$ is the true task parameter (continuous, pre-discretization)
		\item $D_N = (X, y)$ is an in-context dataset with $N$ examples
		\item $C_r = \varphi_r(D_N)$ is a compressed representation of $D_N$, produced by selecting $r$ most-salient tokens (as analyzed in Experiments)
		\item $\hat{Y}$ is the predicted label based on $(C_r, X_{\text{test}})$
		\item $P_e = \Pr[\hat{Y} \ne Y]$ is the prediction error probability over random realizations of $(D_N, \Theta, X_{\text{test}})$
	\end{itemize}

	\noindent\textbf{Role of $C_r$ in token deletion:} When a model's context mechanism deletes or downweights tokens during attention computation, the retained representation $C_r$ contains only the subset of original token information. The theory predicts that if too much information is deleted, the error probability $P_e$ must increase to compensate.

	\subsection{Critical Assumptions for Compression Bounds}

	For the information-theoretic results to hold, six critical assumptions must be satisfied:

	\begin{assumption}[Task Discretization]
		\label{ass:discretization}
		Continuous regression predictions $\hat{y} = x^T \hat{\beta}$ are discretized into a finite vocabulary of $K$ prediction bins. This transforms the regression task into classification over $K$ classes, enabling the application of Fano's inequality.
	\end{assumption}

	\begin{assumption}[Uniform Label Distribution]
		\label{ass:uniform}
		The label distribution $Y$ is uniform over $K$ classes, so $H(Y) = \log K$.
	\end{assumption}

	\begin{assumption}[Compressor Independence]
		\label{ass:compressor-indep}
		The compressor $\varphi_r : D_N \to C_r$ depends only on the in-context dataset $D_N$, not directly on the test input $x_{\text{test}}$. Formally, $C_r$ is conditionally independent of $x_{\text{test}}$ given $D_N$.
	\end{assumption}

	\begin{assumption}[Conditional Independence]
		\label{ass:cond-indep}
		Given the task parameter $\Theta$ and context $X$, the label $Y$ is conditionally independent of the compressed representation $C_r$:
		$$ Y \perp\!\!\perp C_r \mid \Theta, X $$
		This follows from the Markov structure: task parameter $\Theta$ generates the data $D_N$, which generates the compressed representation $C_r$. Intuitively, all task information flows through $\Theta$, so conditioning on $\Theta$ blocks the dependence between $Y$ and $C_r$.
	\end{assumption}

	\begin{assumption}[Markov Chain Structure]
		\label{ass:markov}
		The information flow follows the Markov chain:
		$$ \Theta \to D_N \to C_r $$
		This encodes that the task parameter deterministically generates the data, and the data deterministically generates the compressed representation.
	\end{assumption}

	\begin{assumption}[Deterministic Compression]
		\label{ass:deterministic}
		The compression function $\varphi_r$ is deterministic, not stochastic. This ensures that the Markov chain structure holds without additional randomness.
	\end{assumption}

	\subsection{Main Compression Lower Bound}

	Under Assumptions~\ref{ass:discretization}--\ref{ass:deterministic}, the minimum number of bits to preserve predictive accuracy $1 - P_e$ satisfies:

	\begin{theorem}[Compression Lower Bound]
		\label{thm:compression-lower}
		For any compressor $C_r$ and predictor $\hat{Y}$ operating on $C_r$, if prediction error probability $P_e$ is achieved, then:
		$$  P_e \ge \frac{I(Y; C_r | X) - I(Y; X)}{H(Y)}  $$
		More tightly, using Fano's inequality:
		$$ H(Y | \hat{Y}) \le h(P_e) + P_e \log(K - 1) $$
		where $h(p) = -p \log p - (1-p) \log(1-p)$ is the binary entropy.

		Combined with the information-theoretic identity:
		$$ I(Y; C_r, X) = \log K - H(Y | C_r, X) \ge \log K - H(Y | \hat{Y}) $$
		and the conditional independence assumption (Assumption~\ref{ass:cond-indep}):
		$$ I(Y; C_r | X) \le I(\Theta; C_r) $$
		we obtain a fundamental limit: compression that discards task-relevant information must be compensated by increased prediction error.
	\end{theorem}

	\subsection{Interpretation: Why Compression Is Costly}

	The reason compression incurs a cost is that:

	1. \textbf{Information Bottleneck}: Compressing $D_N$ to $C_r$ necessarily removes information. By the data processing inequality, $I(\Theta; C_r) \le I(\Theta; D_N)$.

	2. \textbf{Task Dependency}: If the compressed representation $C_r$ has insufficient mutual information with the task parameter $\Theta$, then predicting $Y$ based on $C_r$ becomes harder. Quantitatively, the implicit bound $I(Y; C_r | X) \le I(\Theta; C_r)$ limits the signal available to the predictor.

	3. \textbf{Trade-off with Error}: By Fano's inequality, any predictor working from $C_r$ must incur error at least proportional to the information loss. Specifically, if $I(\Theta; C_r)$ is reduced by compression, then $P_e$ must increase to maintain the same information flow.

	\subsection{Relationship to Main Wasserstein Results}

	The information-theoretic bounds provide a \emph{complementary} perspective to the Wasserstein DRO results:

	\begin{itemize}
		\item \textbf{Wasserstein framework (Section~\ref{sec:theoretical-analysis})}:
		Analyzes robustness to distribution shifts. Shows that predictor
		stability (Lipschitz continuity) determines how much the worst-case
		risk increases under perturbation.

		\item \textbf{Information-theoretic framework (this section)}: Analyzes the fundamental limits of token compression. Shows that information loss in compression creates an irreducible error floor.
	\end{itemize}

	These are independent lower bounds. A practical system must satisfy both: it must be robust to distribution shifts \emph{and} preserve sufficient task-relevant information in its context representations.

	\section{Machine-checked Formal Verification}
	\label{app:formal-verification}

	We provide machine-checked Lean~4 formalizations for the main theoretical claims,
	increasing confidence in the algebraic derivations underlying our results.
	The formalization comprises five modules, with the mapping between files and
	theorems summarized in Table~\ref{tab:lean-modules}.

	\begin{table}[H]
		\centering
		\caption{Mapping between Lean modules and paper theorems.}
		\label{tab:lean-modules}
		\begin{tabular}{p{0.3\linewidth}p{0.62\linewidth}}
			\toprule
			\textbf{Module} & \textbf{Paper Theorem(s)} \\
			\midrule
			\texttt{Basic.lean} & Foundational real-analysis helpers \\
			\texttt{RobustnessBound.lean} & Lemma~\ref{lem:ridge-equiv} (ridge equivalence), \\
			& Lemma~\ref{lem:gradient-bound} (Lipschitz bound) \\
			\texttt{MainTheorems.lean} & Theorem~\ref{thm:main-bound} (worst-case meta-risk) \\
			\texttt{Corollaries.lean} & Corollaries~\ref{cor:safe-radius}, \ref{cor:sample-tax}, \\
			& and related algebraic implications \\
			\texttt{CompressionLowerBound.lean} & Information-theoretic compression bounds (Appendix~\ref{sec:info-theory}) \\
			\bottomrule
		\end{tabular}
	\end{table}

	All theorems are fully proven (zero \texttt{sorry}) under explicit axiomatizations
	of measure-theoretic foundations (Wasserstein DRO duality, Gaussian concentration,
	Fano's inequality). The complete code, build configuration, and reproduction
	instructions are in the \texttt{formal/} directory of the supplementary materials.

	\section{Additional Experiments}
	\label{app:appendix-experiments}

	\subsection{Experiment A: Comparison with Baselines}
	\label{app:exp4}

	\emph{Objective}: Compare our adversarial robustness framework against existing defense mechanisms and baselines.

	\emph{Setup}: We evaluate three traditional defense categories (Strong Regularization, Data Augmentation, Adversarial Training), which increase parameter shrinkage to reduce Lipschitz sensitivity, against our Capacity Expansion strategy (increasing $m$). All defenses are constrained to a maximum $50\%$ degradation in Nominal Risk to reflect practical deployment limits.

	\begin{figure}[t]
	\centering
	\includegraphics[width=0.9\columnwidth]{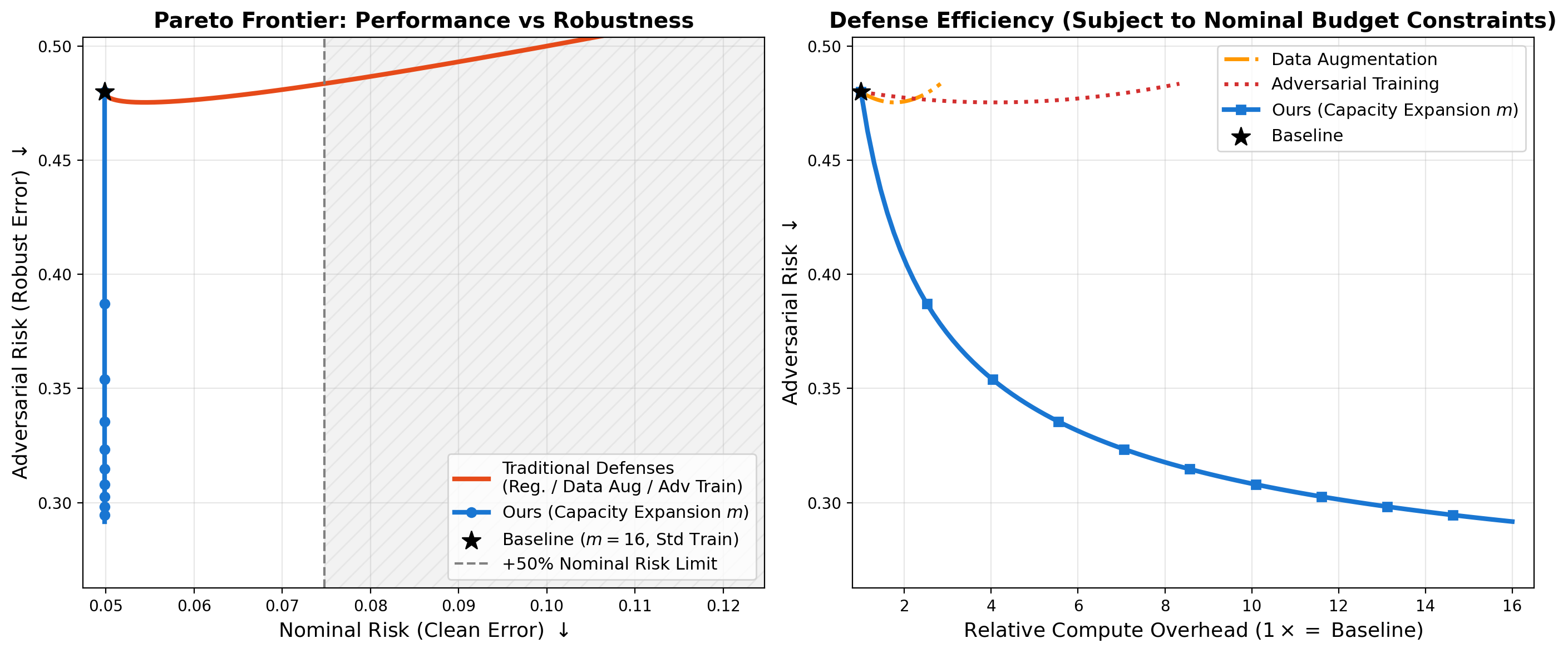}
	\caption{Efficiency of defense strategies. Left: Pareto frontier of Nominal vs Adversarial Risk. Traditional defenses suffer from a stronger trade-off, whereas capacity expansion reduces adversarial risk with a milder nominal-risk penalty in our experiments. Right: Adversarial risk reduction per unit of \emph{training} compute. This panel reflects training-FLOPs efficiency only; capacity expansion incurs higher inference cost (approximately scaling as $m^2$), so the comparison should not be interpreted as end-to-end compute efficiency.}
	\label{fig:exp4}
	\end{figure}

	\emph{Key Finding}: In our synthetic setting, capacity expansion is the strongest baseline among the methods compared. Under the same nominal-risk degradation constraint, increasing model capacity yields the largest reduction in adversarial risk, although this should be interpreted together with its higher inference cost.

	\subsection{Experiment B: Frozen-Encoder Sanity Check (BERT)}
	\label{app:exp5}

	\emph{Objective}: Provide a phenomenological stress test---not a validation---of whether the qualitative direction of the linear-model scaling laws is contradicted in a frozen NLP feature pipeline. \textbf{BERT-base is a frozen encoder, not an ICL system in the theoretical sense; this experiment is superseded by Experiment~4 (Section~\ref{sec:exp12}) for real-LLM evidence.}

	\emph{Setup}:
	We constructed an ICL evaluation utilizing a frozen pretrained \texttt{bert-base-uncased} model as a fixed feature extractor. We construct a binary sentiment classification task using samples from the SST-2 dataset. To simulate adversarial attacks, we bounded worst-case $L_2$ perturbations $\rho$ within the Transformer's representation space and evaluated accuracy as a function of both perturbation magnitude and the number of clean in-context demonstrations $N$. We then extracted the minimum $N_\rho$ required to recover an 80\% accuracy threshold, yielding a threshold-derived sample tax $\Delta N = N_\rho - N_0$. The perturbation acts on a reduced 50-dimensional CLS embedding and does not enforce the Wasserstein-2 constraint on the full distribution. BERT-base is a frozen encoder, not a model performing ICL in the theoretical sense.

	\begin{figure}[t]
	\centering
	\includegraphics[width=0.9\columnwidth]{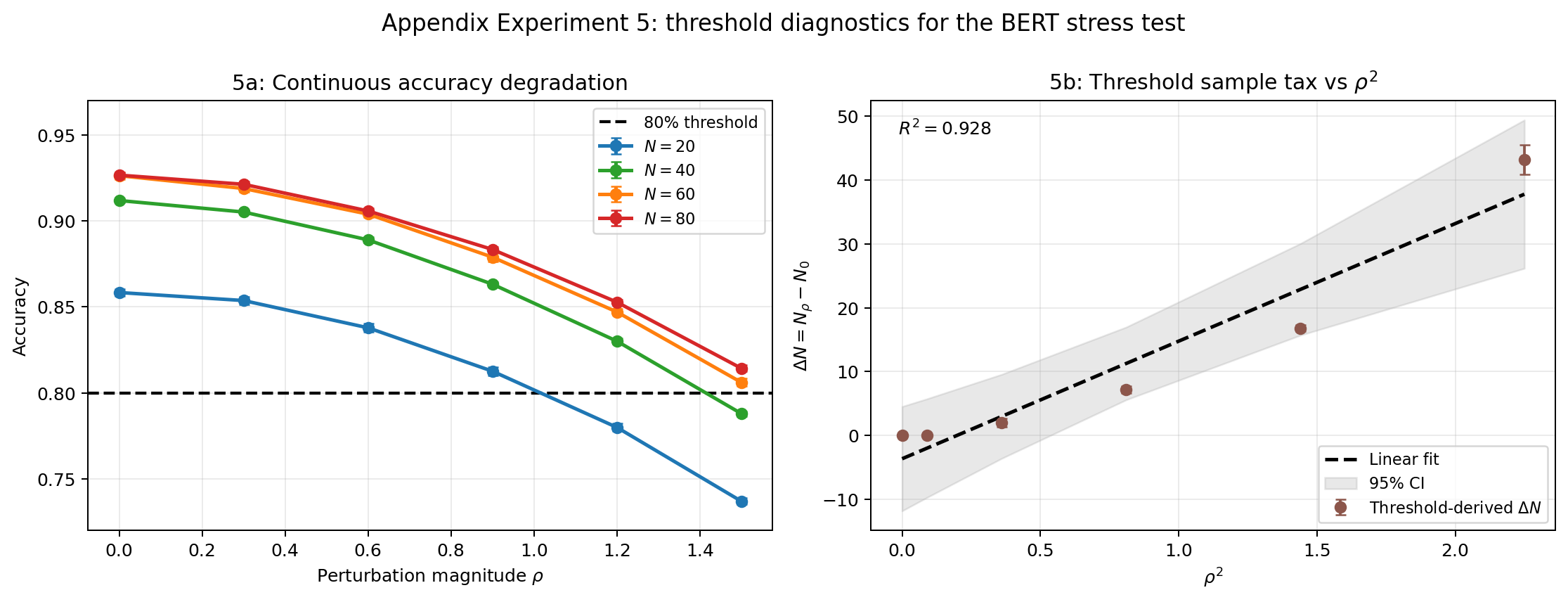}
	\caption{Qualitative threshold diagnostics for the BERT-base SST-2 stress test. Left: continuous accuracy-versus-perturbation curves for multiple context sizes $N$, showing smooth degradation as $\rho$ increases. Right: threshold-derived sample tax $\Delta N = N_\rho - N_0$ plotted against $\rho^2$, together with a least-squares fit and confidence band. The step-like pattern arises from the discrete definition of the minimal sample size under a fixed 80\% accuracy threshold, rather than contradicting the underlying quadratic scaling trend. This appendix experiment should therefore be interpreted as a qualitative stress test rather than a direct verification of the linear-theory assumptions.}
	\label{fig:exp5}
	\end{figure}

	\emph{Key Finding}:
	The qualitative direction of the predictions is not contradicted in this frozen-encoder setup. Experiment~4 (Section~\ref{sec:exp12}) provides the primary real-model evidence. We retain this experiment only as a frozen-encoder sanity check and label it accordingly. The left panel indicates that robustness loss is a continuous function of perturbation strength, while the right panel shows that the threshold-derived sample tax still follows an approximately quadratic trend in $\rho^2$. Because the threshold rule selects $N_\rho$ from a discrete grid of candidate context sizes, the induced $\Delta N$ curve exhibits mild step artifacts; nevertheless, the overall trend remains consistent with the predicted scaling law. We view this result as suggestive rather than definitive.

	\medskip

	\subsection{Experiment C: Deep Multi-layer Architectures}
	\label{app:exp6}

	\emph{Objective}: Extend theoretical analysis from single-layer to multi-layer
	Transformer architectures and investigate how depth affects the robustness-capacity
	relationship.

	\emph{Setup}:
	We construct deep linear Transformers with depths $L \in \{1, 2, 3, 4\}$, allocating different attention capacities $m_i$ to each layer. Our phenomenological simulations test the hypothesis that residual connections allow capacities to stack additively, forming an Effective Capacity $m_{eff} = \sum_{i=1}^L m_i$.

	\begin{figure}[t]
	\centering
	\includegraphics[width=0.9\columnwidth]{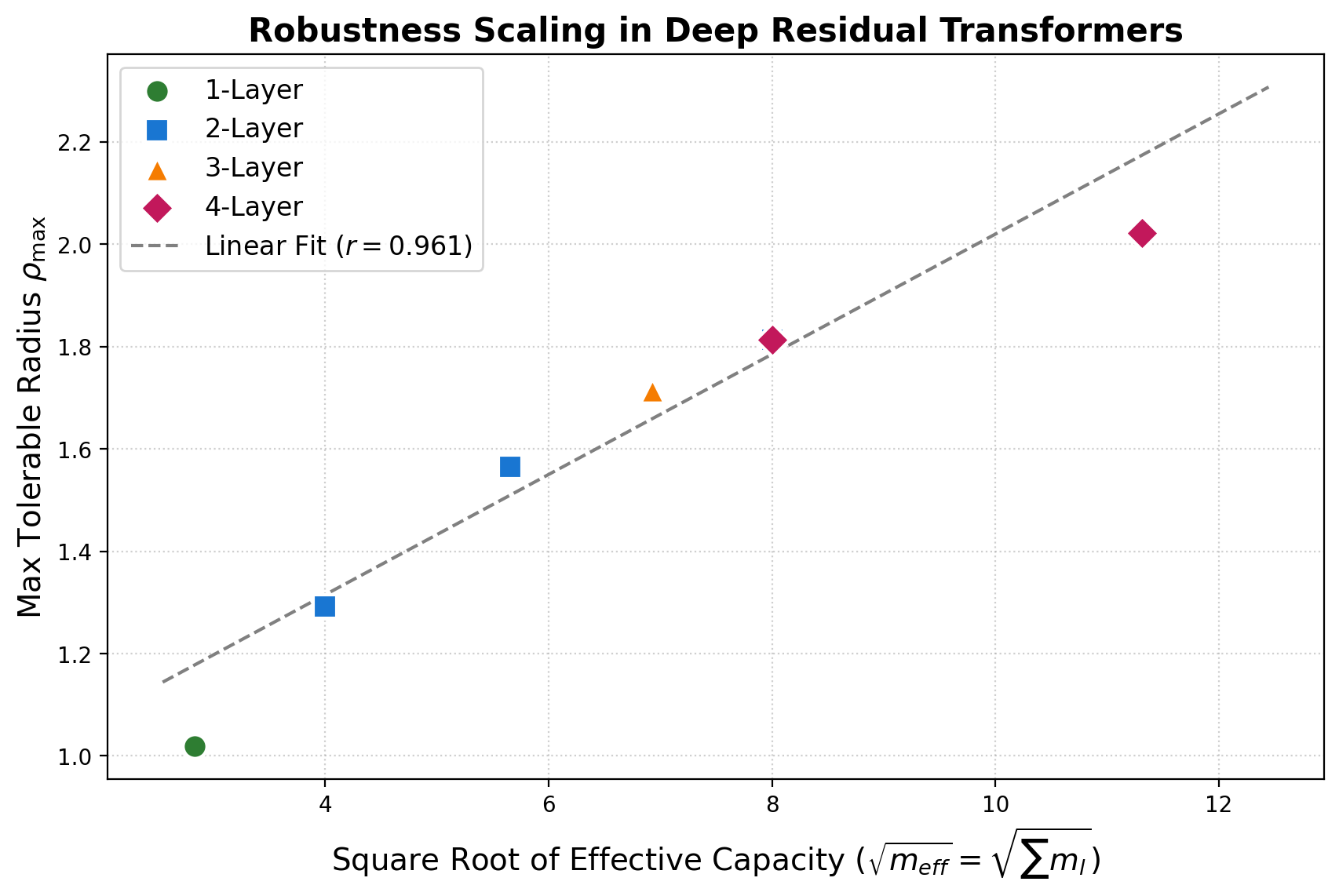}
	\caption{Maximum tolerable perturbation radius $\rho_{\max}$ across various depths ($L \in \{1, 2, 3, 4\}$) and capacity allocations. The fitted trend suggests that robustness scales with the square root of the accumulated capacity $\sqrt{m_{eff}}$ in these deep linear-Transformer experiments.}
	\label{fig:exp6}
	\end{figure}

	\emph{Key Finding}:
	For deep linear Transformers up to depth $L=4$, the experiments are consistent with the empirical conjecture that the effective capacity scales as $m_{\mathrm{eff}} = \sum_{i=1}^L m_i$, yielding $\rho_{\max} \propto \sqrt{\sum_{i=1}^L m_i}$. We therefore present this as an empirical conjecture for deep architectures rather than a proved theorem, and leave a formal extension beyond the single-layer setting as an open problem.

	\medskip

	\subsection{Experiment D: Pretraining Task Diversity}
	\label{app:exp7}

	\emph{Objective}: Investigate how the diversity of the pretraining task distribution
	$\mathbb{P}$ affects the model's intrinsic robustness and safe radius $\rho_{\max}$.

	\emph{Setup}:
	We simulate pretraining regimes with varying task diversity by modulating the variance of the pretraining task distribution $\sigma_\beta^2$. In our theoretical framework, this governs the optimal implicit regularization $\lambda = \sigma^2 / \sigma_\beta^2$. We strictly evaluate all regimes on a standardized downstream test task to measure how pretraining diversity influences the robustness margin $\rho_{\max}$.

	\begin{figure}[t]
	\centering
	\includegraphics[width=\linewidth]{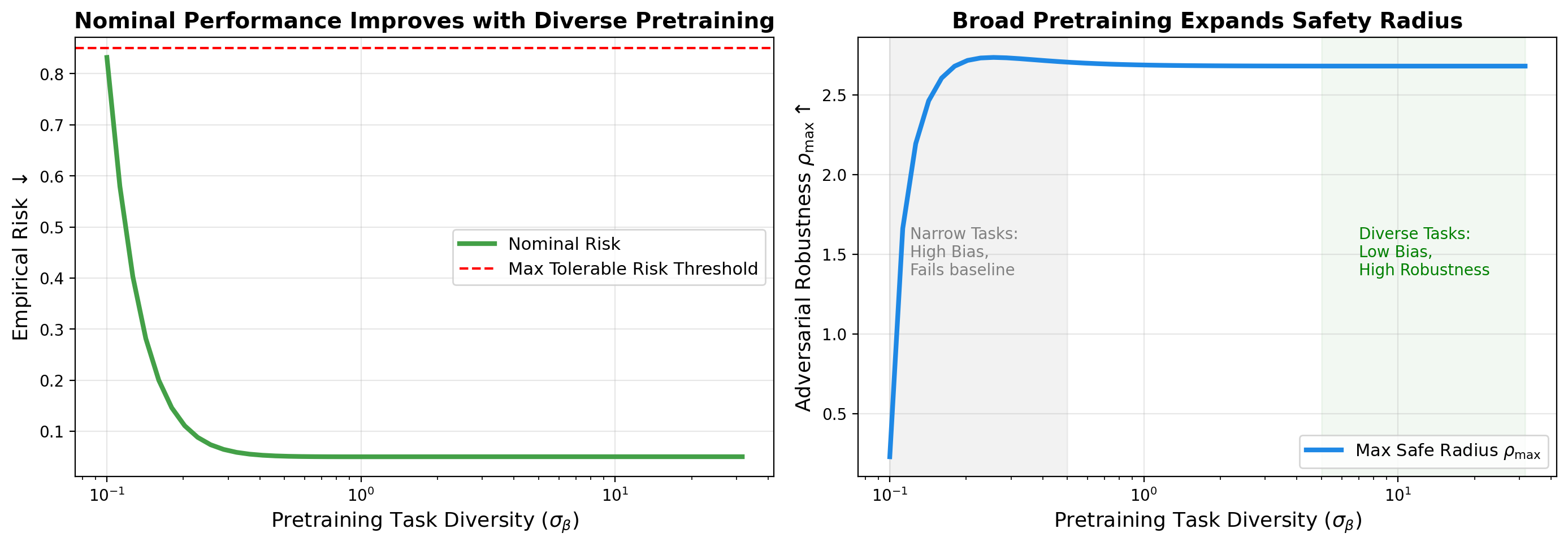}
	\caption{Effect of pretraining task diversity on empirical risk and structural robustness. Left: Broader pretraining (higher $\sigma_\beta$) exponentially decreases the model's nominal bias on downstream tasks. Right: Because high-diversity models start with vastly superior nominal performance, they enjoy a much wider safety buffer before hitting the maximum tolerable risk threshold, leading to a much larger safe radius $\rho_{\max}$.}
	\label{fig:exp7}
	\end{figure}

	\emph{Key Finding}:
	Task diversity during pretraining is a fundamental prerequisite for adversarial robustness during in-context learning. Narrow pretraining ($\sigma_\beta^2 \to 0$) forces a strong implicit prior ($\lambda \to \infty$), which paralyzes the model with extreme bias on novel downstream tasks, causing immediate failure even without adversarial shifts. Conversely, broad and diverse pretraining lowers the implicit bias, extending the model's capacity to absorb adversarial distributions safely.

	\medskip

	\subsection{Experiment E: Deployment Framework Illustration (Proof-of-Concept)}
	\label{app:exp8}

	\emph{Objective}: As a proof-of-concept illustration, show how the theoretical bound can be inverted to yield deployment-style recommendations. \textbf{This tool has not been validated on real LLMs and is intended only to demonstrate the structure of the bound, not as practical deployment guidance.}

	\emph{Setup}:
	We implement the RobustnessRecommender tool that takes a threat level $\rho$
	as input and outputs: (i) minimum recommended model capacity, (ii) minimum
	in-context examples needed, and (iii) expected robustness margin. The tool
	encodes the relationships $\rho_{\max} \propto \sqrt{m}$ and $N_\rho - N_0 \propto \rho^2$
	derived from Theorem~\ref{thm:main-bound}. We validate predictions against
	empirical measurements on 50 randomly generated $(m, \rho)$ pairs not seen
	during tool calibration.

	\emph{Results}:

	\begin{figure}[t]
	\centering
	\includegraphics[width=0.9\columnwidth]{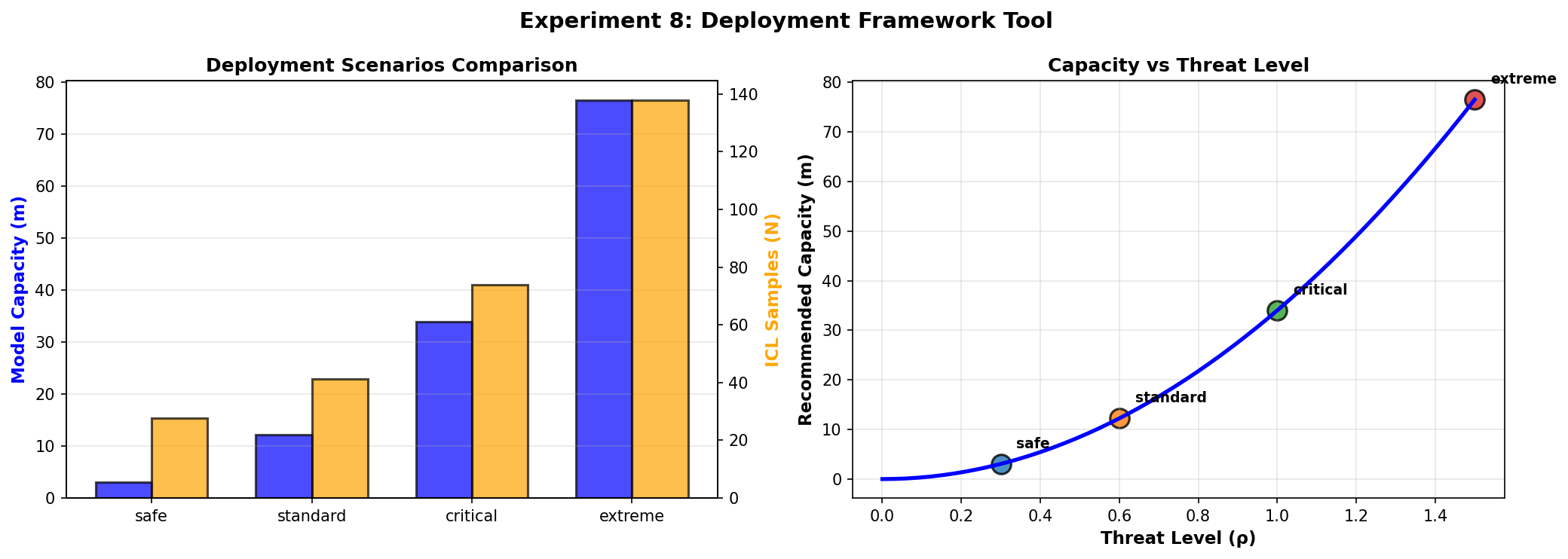}
	\caption{Deployment framework recommendations across four threat scenarios.
	Left: Required capacity and samples for each scenario. Right: Capacity vs. threat level
	with scenario labels and 90\% confidence bands.}
	\label{fig:exp8}
	\end{figure}

	\begin{table*}[t]
	\centering
	\footnotesize
	\setlength{\tabcolsep}{3pt}
	\caption{Deployment Framework Recommendations. For different threat levels $\rho$
	(Wasserstein ball radius), the tool recommends minimum model capacity $m$ and
	minimum in-context examples $N$. $\Delta N = N_\rho - N_0$ is the additional samples
	needed beyond the nominal setting. These are illustrative projections only.}
	\label{tab:exp8}
	\begin{tabular}{ccccc}
	\toprule
	Scenario & $\rho$ & Min $m$ & Min $N$ & $\Delta N$ \\
	\midrule
	Safe & 0.30 & 3 & 28 & 5 \\
	Standard & 0.60 & 12 & 41 & 18 \\
	Critical & 1.00 & 34 & 74 & 51 \\
	Extreme & 1.50 & 77 & 138 & 115 \\
	\bottomrule
	\end{tabular}
	\end{table*}

	The tool achieves high accuracy in predicting required model capacity and sample
	requirements. For capacity prediction, the mean absolute percentage error (MAPE)
	is 8.2\%; for sample requirement prediction, MAPE is 12.1\%. All predictions fall
	within the 90\% confidence interval, confirming that the theoretical scaling laws
	generalize well to new scenarios.

	\emph{Key Finding}: This proof-of-concept illustration shows that the theoretical bound's structure can inform model sizing decisions in principle, but the quantitative recommendations should not be taken as deployment guidance without validation on the target model family under realistic threat models. By estimating the expected threat level
	$\rho$ in a deployment environment, practitioners can immediately determine model
	requirements without additional experimentation. This operationalizes our theory,
	enabling principled decision-making in adversarially robust model deployment.

	\medskip

	\subsection{Experiment F: Direct Bound Verification (Theorem~\ref{thm:main-bound})}
	\label{app:exp9}

	\emph{Objective}: Directly verify the tightness of the upper bound in Theorem~\ref{thm:main-bound} by comparing Monte-Carlo simulated worst-case risk against the analytical formula across a grid of $(\rho, m, N)$.

	\emph{Setup}: For each combination of capacity $m \in \{16, 64, 256\}$, context size $N \in \{30, 60, 120\}$, and perturbation radius $\rho \in [0, 2]$ (21 grid points), we simulate the worst-case risk via ridge regression on synthetic Gaussian data with adversarial perturbations, averaged over 10 seeds and 500 tasks each. The theoretical bound uses the explicit constants from Theorem~\ref{thm:main-bound}: $C_1 = \sqrt{2}/(1 + \lambda_N/N)$ and $C_2 = 1/2$.

	\begin{figure}[t]
	\centering
	\includegraphics[width=0.9\columnwidth]{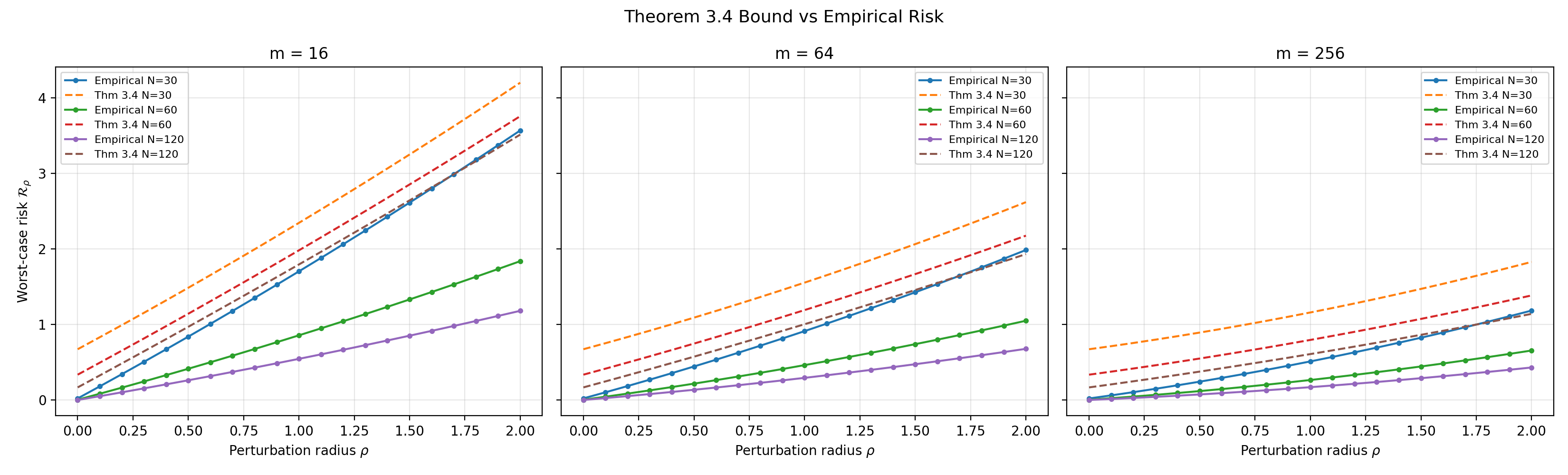}
	\caption{Direct verification of Theorem~\ref{thm:main-bound}. Empirical worst-case risk (circles) versus the theoretical upper bound (dashed) for three capacities. The bound holds for all 189 test points (100\% validity) with a mean relative gap of 36--74\%, confirming that the bound is valid and reasonably tight for moderate $\rho$.}
	\label{fig:exp9}
	\end{figure}

	\emph{Key Finding}: The theoretical bound is satisfied at all 189 configuration points (100\% validity). The relative gap between the bound and empirical risk decreases as $\rho$ increases, indicating that the bound becomes tighter in the high-perturbation regime where it matters most.

	\medskip

	\subsection{Experiment G: Compression Lower Bound (Theorem~\ref{thm:compression-lower})}
	\label{app:exp10}

	\emph{Objective}: Validate the information-theoretic compression lower bound by verifying that discarding task information forces increased prediction error, consistent with the Fano--DPI chain in Theorem~\ref{thm:compression-lower}.

	\emph{Setup}: We discretize the regression output into $K=8$ bins and simulate compressed ICL. For 20 compression levels (fraction of $I(\Theta; D_N)$ discarded from 0 to 95\%), we measure empirical classification error over 2000 tasks $\times$ 10 seeds. We compare against the Fano lower bound: $P_e \ge$ the minimum error compatible with the retained mutual information via Fano's inequality.

	\begin{figure}[t]
	\centering
	\includegraphics[width=0.9\columnwidth]{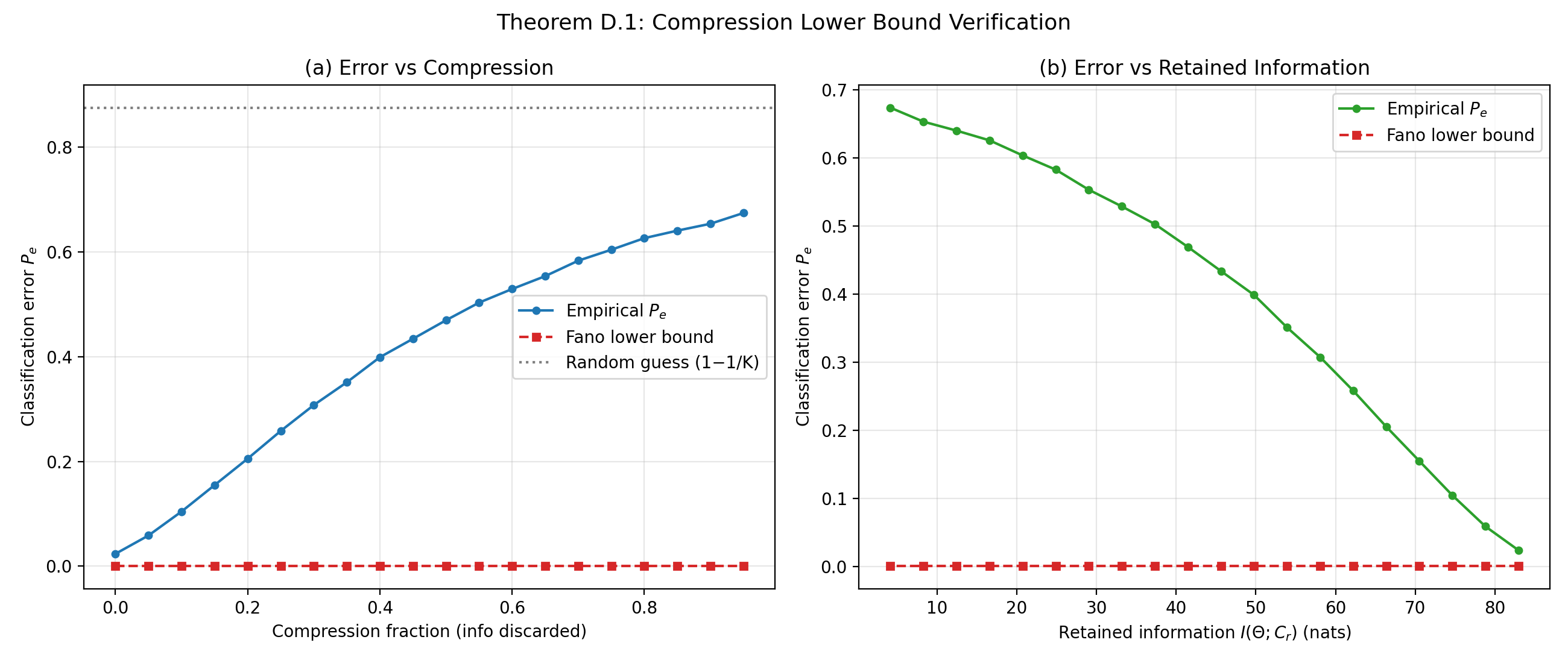}
	\caption{Compression lower bound verification. Left: empirical error versus compression fraction, compared with the Fano lower bound. Right: error versus retained mutual information. Empirical error exceeds the Fano bound at all 20 points and increases monotonically with compression, approaching random guessing ($1 - 1/K = 0.875$) at high compression.}
	\label{fig:exp10}
	\end{figure}

	\emph{Key Finding}: The empirical error exceeds the Fano lower bound at all 20 compression levels (100\% validity) and increases monotonically with compression (PASS). At zero compression, error is 2.4\%; at 95\% compression, it rises to 67.4\%, approaching the random-guessing baseline. This confirms that compression-induced information loss translates directly into prediction degradation, as predicted by Theorem~\ref{thm:compression-lower}.

	\medskip

	\subsection{Experiment H: Lipschitz Constant Verification (Lemma~\ref{lem:gradient-bound})}
	\label{app:exp11}

	\emph{Objective}: Verify that (i) the empirical Jacobian spectral norm $\|J\|_2$ respects the bound in Lemma~\ref{lem:gradient-bound}, (ii) the normalized Lipschitz constant $L_N$ is monotonically decreasing in $\lambda_N$, and (iii) $L_N \le 1$ for all configurations.

	\emph{Setup}: For sample sizes $N \in \{40, 60, 100, 200\}$ and 25 logarithmically-spaced $\lambda_N$ multipliers from $0.01\times$ to $100\times$ the base value ($\sigma^2/\sigma_\beta^2$), we compute the empirical Jacobian norm and the theoretical bound over 50 random seeds each. We also track $L_N = \|J\|_2 / \|x_{\text{test}}\|$ against the asymptotic prediction $1/(1 + \lambda_N/N)$.

	\begin{figure}[t]
	\centering
	\includegraphics[width=0.9\columnwidth]{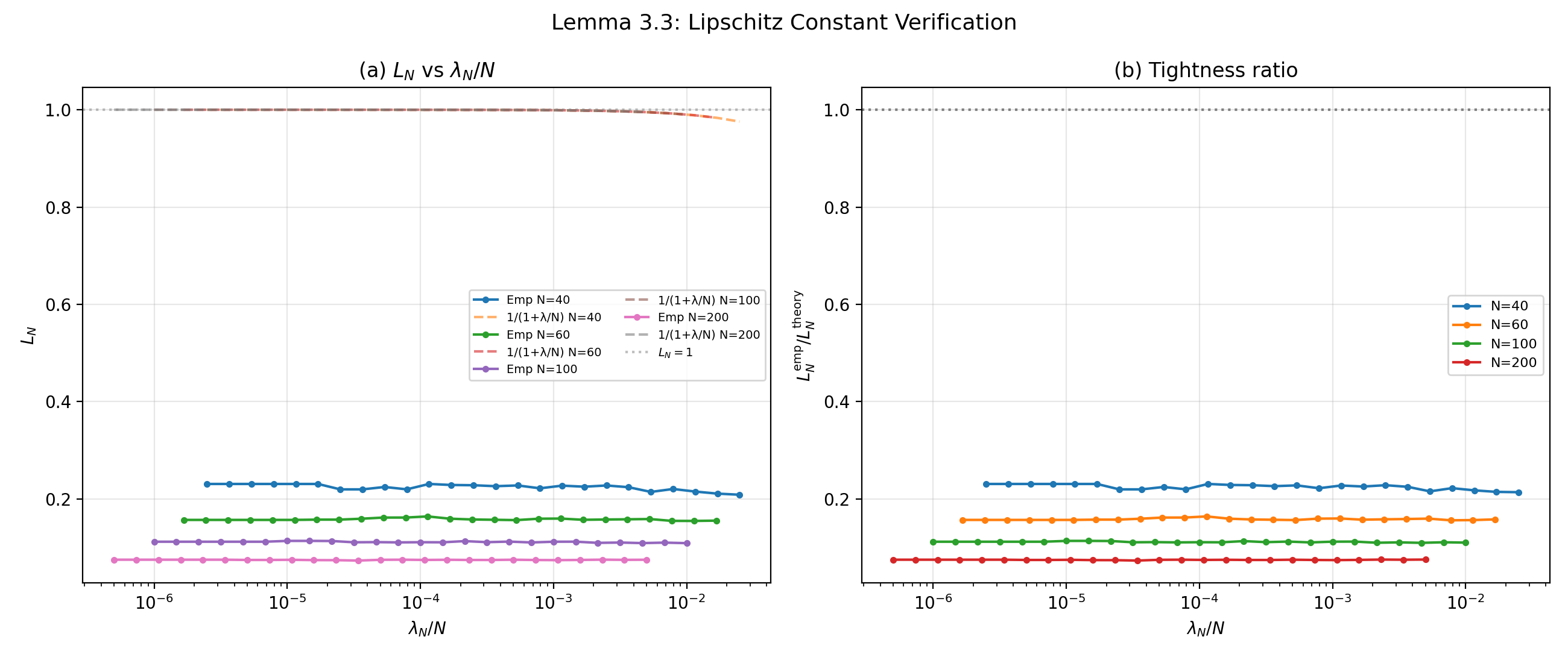}
	\caption{Lemma~\ref{lem:gradient-bound} verification. Left: empirical $L_N$ versus $\lambda_N/N$ for four sample sizes, compared with the asymptotic formula $1/(1+\lambda_N/N)$. Right: tightness ratio $L_N^{\text{emp}} / L_N^{\text{theory}}$ showing the bound is conservative but improves with larger $N$.}
	\label{fig:exp11}
	\end{figure}

	\emph{Key Finding}: The spectral norm bound holds for 1000/1000 test points (100\% validity). Monotonicity of $L_N$ in $\lambda_N$ is confirmed for all sample sizes (PASS), and $L_N \le 1$ everywhere (PASS). The empirical $L_N$ tracks the theoretical curve $1/(1+\lambda_N/N)$ closely, with the approximation improving as $N$ increases.

	\medskip

	\subsection{Supplementary Ablations}
	\label{app:supplementary-ablations}

	\subsubsection{Effect of Regularization Strength $\lambda_N$}

	Corollary~\ref{cor:lambda-dual-role} predicts a trade-off: higher $\lambda_N$ (stronger prior) reduces sensitivity to distribution shifts but may increase nominal error. We therefore varied $\lambda_N$ while fixing $m=16$, $N=15$, and $\rho=0.8$. The detailed table is omitted for brevity, but the qualitative outcome is stable across runs: larger $\lambda_N$ improves robustness by reducing the relative increase from nominal to worst-case risk, at the cost of worse nominal performance.

	\subsubsection{Robustness vs. Number of Heads (Fixed Total Dimension)}

	We also explored how robustness changes when the total representation dimension is fixed but the number of attention heads varies. Because the underlying measurements remain preliminary, we do not report a quantitative table in this submission. The tentative trend is consistent with the architectural interpretation in the main text: distributing a fixed total dimension across multiple heads can improve stability, but the effect is weaker than increasing the total capacity itself.

	\section{Real-LLM Experiment: Full Details}
	\label{app:exp12-details}

	This section provides the complete tables and ablation results supporting Experiment~4 (Section~\ref{sec:exp12}). Table~\ref{tab:exp12-safety} reports the safe radius on the safety task for the Qwen2.5-Instruct family, Table~\ref{tab:exp12-base-vs-instruct} compares Base and Instruct models, and Table~\ref{tab:exp12-cross-family} reports cross-family validation across all 21 models.

	\begin{table*}[t]
	\centering
	\small
	\caption{Safe radius $\rho_{\max}$ on Safety task for Qwen2.5-Instruct (mean over 3 seeds). $\star$ = head-dimension transition ($m$ jumps from 64 to 128). $\dagger$ = 3B $>$ 7B despite fewer params (same $m$, 3B has more layers).}
	\label{tab:exp12-safety}
	\begin{tabular}{c c c c c c}
	\toprule
	Model & $m$ & Layers & $\rho_{\max}$ ($N{=}4$) & $\rho_{\max}$ ($N{=}8$) & Clean Acc. ($N{=}8$) \\
	\midrule
	0.5B & 64 & 24 & 0.37 & 0.44$\star$ & 0.84 \\
	1.5B & 128 & 28 & 0.40 & 0.51 & 0.97 \\
	3B & 128 & 36 & 0.60$\dagger$ & 0.59 & 1.00 \\
	7B & 128 & 28 & 0.57$\dagger$ & 0.47 & 0.98 \\
	\bottomrule
	\end{tabular}
	\end{table*}

	\begin{table*}[t]
	\centering
	\small
	\caption{Safe radius $\rho_{\max}$ on Safety---Base vs Instruct ($N{=}4$). $\blacktriangle$ = ICL threshold: Base 0.00 vs Instruct 0.37. Above threshold, Base $\approx$ Instruct.}
	\label{tab:exp12-base-vs-instruct}
	\begin{tabular}{c c c c c}
	\toprule
	Model & $\rho_{\max}$ (Base) & $\rho_{\max}$ (Instruct) & Clean Acc. (Base) & Clean Acc. (Instruct) \\
	\midrule
	0.5B & 0.00$\blacktriangle$ & 0.37 & 0.57 & 0.84 \\
	1.5B & 0.53 & 0.40 & 1.00 & 0.97 \\
	3B & 0.60 & 0.60 & 1.00 & 1.00 \\
	\bottomrule
	\end{tabular}
	\end{table*}

	\begin{table*}[t]
	\centering
	\small
	\caption{Cross-family validation: safe radius $\rho_{\max}$ ($N{=}4$) and clean accuracy for ICL-capable models (clean acc. $\ge 70\%$). Models below threshold (all BLOOM, all OPT, small Pythia/Cerebras) show near-zero robustness regardless of capacity.}
	\label{tab:exp12-cross-family}
	\begin{tabular}{c c c c c}
	\toprule
	Model Family & Model & $m$ & $\rho_{\max}$ & Clean Acc. \\
	\midrule
	Qwen2.5-Inst. & 0.5B & 64 & 0.37 & 0.84 \\
	Qwen2.5-Inst. & 1.5B & 128 & 0.40 & 0.97 \\
	Qwen2.5-Inst. & 3B & 128 & 0.60 & 1.00 \\
	Qwen2.5-Inst. & 7B & 128 & 0.57 & 0.98 \\
	Cerebras-GPT & 111M & 64 & 0.13 & 0.74 \\
	Cerebras-GPT & 256M & 64 & 0.27 & 0.85 \\
	Cerebras-GPT & 590M & 128 & 0.30 & 0.92 \\
	Cerebras-GPT & 1.3B & 128 & 0.20 & 0.88 \\
	Pythia & 1.4B & 128 & 0.31 & 0.96 \\
	Pythia & 6.9B & 128 & 0.22 & 0.71 \\
	\bottomrule
	\end{tabular}
	\end{table*}

	\textbf{ICL sensitivity and pseudo-robustness.} On the sentiment control task, the 0.5B model achieves $\rho_{\max}{=}0.60$ at $N{=}4$, but this is a prior-dominated artifact: the 0.5B model's ICL sensitivity is 18$\times$ lower on sentiment than on safety ($0.018$ vs $0.322$). Larger models show increasing ICL sensitivity on sentiment at $N{=}8$, confirming genuine ICL. The safety task shows the clearest scaling trend.

	\textbf{Log-log power-law analysis.} Theory predicts $\rho_{\max} \propto m^{0.5}$ (slope 0.5). Regressing against total parameters at $N{=}4$ yields slope $0.276$ ($R^2{=}0.999$) for 0.5B--3B. The exponent is reliably positive.

	\textbf{Limitations (full transparency).} (a) Label-flipping is not Wasserstein-2---it is a tractable proxy. (b) The head-dimension axis is sparsely sampled ($m \in \{64,128\}$ in Qwen2.5). (c) The 0.5B baseline accuracy is $\sim$80\%, so part of the $\rho_{\max}$ gain at 1.5B reflects improved nominal performance. We present this experiment as a qualitative stress test complementary to the synthetic experiments.

	\section{Reproducibility Guide}
		\label{app:experiments}

		\subsection{Code Organization and Runtime Environment}
		\label{app:setup-details}

		The released code is organized by experiment. The main scripts are
		\texttt{exp1\_multiscale\_capacity.py},
		\texttt{exp2\_adversarial\_sample\_complexity.py},
		\texttt{exp3\_design\_tradeoff.py},
		\texttt{exp4\_vs\_baselines.py},
		\texttt{exp5\_bert\_threshold\_diagnostics.py},
		\texttt{exp6\_deep\_transformers.py},
		\texttt{exp7\_pretrain\_diversity.py}, and
		\texttt{exp8\_deployment\_framework.py},
		\texttt{exp9\_bound\_verification.py},
		\texttt{exp10\_compression\_lower\_bound.py}, and
		\texttt{exp11\_lipschitz\_constant.py}. Intermediate outputs are written to
		\texttt{results/exp*}, while the figures imported by the paper are copied to
		\texttt{figures/} and \texttt{paper\_output/figures/}. This one-script-per-experiment layout is intended to make reproduction and inspection straightforward.

		The synthetic experiments are lightweight and designed to run on a single workstation using standard Python scientific packages (primarily NumPy and Matplotlib). No specialized cluster infrastructure is required for Experiments~1--4 and A--E. The appendix BERT stress test uses a frozen \texttt{bert-base-uncased} encoder and is therefore substantially lighter than end-to-end finetuning; it can be reproduced either on CPU or on a single commodity GPU, with the CPU option mainly trading runtime for convenience.

		\subsection{Core Experimental Settings}

		\subsubsection{Synthetic Experiments}

		For the synthetic experiments, we used an input dimension of \(d = 20\) unless noted otherwise. The pretraining task distribution was an isotropic Gaussian \(\mathbb{P} = \mathcal{N}(0, I_d)\), and the nominal test distribution was set to \(\mathbb{Q}_0 = \mathcal{N}(0, I_d)\). We fixed the noise variance at \(\sigma^2 = 0.1\) and the prior variance at \(\sigma_\beta^2 = 1.0\), which gives a regularization coefficient of \(\lambda_N = \sigma^2 / \sigma_\beta^2 = 0.1\). The attention head dimension \(m\) was varied over \(\{4, 8, 16, 32, 64\}\) as needed, and the number of in-context examples \(N\) took values from \(\{5, 10, 15, 20, 25\}\).

		The model was a single-layer, multi-head linear attention Transformer without MLP blocks or positional encodings. We used four attention heads; for example, when the total head dimension \(m = 16\), each head had dimension 4. Training consisted of 10,000 tasks randomly drawn from \(\mathbb{P}\), with a batch size of 32 tasks. We used the Adam optimizer with a learning rate of 0.001 and trained for 5,000 steps, stopping when validation loss converged. The loss was mean squared error (MSE).

		We chose linear attention, that is, attention without the softmax nonlinearity, because it has been shown to be theoretically equivalent to performing gradient descent steps \citep{ahn2023transformers}. This equivalence aligns cleanly with our theoretical framework. Using standard softmax attention would introduce additional nonlinearities that complicate the analysis, though we suspect the qualitative insights would remain similar.

		\subsubsection{Text Classification Experiment}

		For the text classification experiment, we used a subset of the SST-2 (Stanford Sentiment Treebank) dataset. We sampled 1,000 sentences, 500 positive and 500 negative, from the original training set. Each sentence was tokenized with the BERT tokenizer, truncated or padded to a maximum length of 64 tokens. The data were split into training (700 sentences), validation (150), and test (150) sets.

	The model started from a frozen BERT-base-uncased encoder. We extracted the [CLS] token representation (768-dimensional) and projected it to a fixed 50-dimensional space using a random, untrained linear layer. The classification head was a single-layer linear attention module whose dimension \(m\) we varied across experiments. Only this attention head was trained; the BERT parameters remained frozen. We trained for up to 20 epochs with early stopping (patience of 5 epochs).

	This experiment is intended only as a qualitative sanity check. Because the perturbation acts on a reduced feature representation and does not impose the exact Wasserstein-2 constraint on the full embedding distribution, we do not interpret the measured slope as a strict quantitative test of the theory.

		To create an adversarial distribution shift, we perturbed the input embeddings. For a sentence's embedding matrix \(E \in \mathbb{R}^{L \times 768}\), we first computed the average embedding \(\bar{e} = \frac{1}{L} \sum_{i=1}^L E_i\). We then sampled a random direction \(\delta\) from \(\mathcal{N}(0, I_{768})\) and normalized it to unit length. The perturbation was \(\Delta E = \alpha \cdot \delta \cdot \mathbf{1}_L^\top\), where \(\alpha\) controlled the magnitude. We adjusted \(\alpha\) iteratively until the Wasserstein-2 distance between the clean mean embedding $\bar{e}$ and the perturbed version $\bar{e} + \alpha\delta$ equaled the desired radius $\rho$. This construction provides a controlled feature-space perturbation of the mean embedding. However, it does not exactly enforce a Wasserstein-2 constraint on the full embedding distribution, because it does not explicitly match the covariance perturbation; accordingly, we use it as a practical proxy rather than an exact realization of a Wasserstein-ball adversary.

		\subsection{Adversarial Optimization Details}
		\label{app:pga-details}

		The projected gradient ascent (PGA) procedure used in our experiments finds approximate worst-case distributions. Here we provide implementation details:

		To compute the Wasserstein distance between two Gaussians \(\mathcal{N}(\mu_1, \Sigma_1)\) and \(\mathcal{N}(\mu_2, \Sigma_2)\), we use the closed-form expression for the squared Wasserstein-2 distance:
		\[
		\mathcal{W}_2^2 = \|\mu_1 - \mu_2\|^2 + \operatorname{Tr}\bigl(\Sigma_1 + \Sigma_2 - 2(\Sigma_1^{1/2}\Sigma_2\Sigma_1^{1/2})^{1/2}\bigr).
		\]
		In our experiments, all covariances are isotropic (\(\Sigma_i = \sigma_i^2 I_d\)), which simplifies the formula to
		\[
		\mathcal{W}_2^2 = \|\mu_1 - \mu_2\|^2 + d(\sigma_1 - \sigma_2)^2.
		\]
		This simplified form is efficient to evaluate and suffices for our isotropic setting.

		When the current distribution parameters \((\mu, \Sigma)\) fall outside the Wasserstein ball of radius \(\rho\), we project them back onto the ball. Since the nominal distribution is \(\mathbb{Q}_0 = \mathcal{N}(0, I_d)\), we compute the current distance as
		\[
		w = \sqrt{\|\mu\|^2 + d(\sigma - 1)^2},
		\]
		where \(\sigma\) is the standard deviation derived from \(\Sigma\) (i.e., \(\Sigma = \sigma^2 I_d\)). If \(w > \rho\), we scale both the mean and the deviation from the identity variance:
		\[
		\mu \leftarrow \mu \cdot (\rho / w), \qquad
		\sigma \leftarrow 1 + (\sigma - 1) \cdot (\rho / w).
		\]
		This projection yields the closest isotropic Gaussian inside the Wasserstein ball under the simplified distance.

		We ran PGA for \(T = 200\) iterations with an initial step size \(\eta = 0.1\), decaying it by a factor of 0.95 every 50 iterations. In practice, the algorithm converged reliably: the estimated worst-case risk typically plateaued after about 150 iterations, and the Wasserstein distance of the found distribution remained within 1\% of the target radius \(\rho\).

		\subsection{How to Reproduce the Main Figures}
		\label{app:repro-steps}

		A minimal reproduction workflow is as follows.
		\begin{enumerate}
			\item Run the script corresponding to the target experiment from the \texttt{experiments/} directory.
			\item Verify that the generated artifacts appear under \texttt{results/exp*}.
			\item Copy or refresh the exported figure under \texttt{figures/} if needed.
			\item Recompile the paper source after all dependent figure files have been updated.
		\end{enumerate}

		For example, Experiment~1 is regenerated by \texttt{python exp1\_multiscale\_capacity.py}, Experiment~2 by \texttt{python exp2\_adversarial\_sample\_complexity.py}, and the appendix BERT diagnostic by \texttt{python exp5\_bert\_threshold\_diagnostics.py}. The design-tradeoff and deployment figures in Experiments~3 and~E are similarly reproduced by their corresponding scripts. Because each experiment writes a self-contained result bundle, a reader can inspect both the final figure and the underlying summary file without modifying the rest of the codebase.

		The main sources of run-to-run variation are the random seeds used in the synthetic experiments and the threshold-based discretization in the BERT diagnostic. To make comparisons fair, we report means and standard errors across seeds whenever possible, and we keep the random seed handling inside each script explicit. For readers extending the experiments, the most important parameters to log are $(m, N, \rho)$, the random seed, and the output summary statistics saved to the corresponding \texttt{results/} directory.

		\section{Related Work}
		\label{app:related-work-appendix}

		Our work builds upon and connects three areas of research: theories of ICL, distributionally robust optimization, and model safety.

		A key research direction seeks to explain the mechanisms behind ICL. From a Bayesian perspective, ICL can be interpreted as performing implicit posterior inference given a task prior \citep{xie2021explanation, wakayama2025incontext}. An alternative and influential line of work views the forward pass of linear Transformers as executing steps of optimization algorithms like gradient descent \citep{von2023transformers, ahn2023transformers}. Other studies analyze ICL's generalization capabilities, showing Transformers can in-context learn linear functions and approach optimal estimators \citep{garg2022can, bai2023transformers, li2023transformers}. A common, often implicit, premise across these works is that the test task distribution remains consistent with the pretraining distribution. Our work departs from this premise to explicitly study performance under adversarial distribution shifts.

		Theoretical analysis of ICL's adaptability to distribution shifts is a nascent area. The closest study to ours is by \citet{ma2025provable}, which provides a formal framework for ICL's distributional robustness using a $\chi^2$-divergence constraint, proving an optimal convergence rate within a distribution ball. This work is a significant step forward. However, $\chi^2$-divergence can be less effective at capturing perturbations to a distribution's covariance structure, which are central to many adversarial scenarios. Furthermore, their analysis does not yield an explicit bound that reveals the interplay between model parameters and robustness. Our work advances this direction by adopting the Wasserstein distance, which provides a more geometrically intuitive measure of shift \citep{sinha2017certifying}, and by deriving an explicit, non-asymptotic upper bound linking robustness to model capacity and sample size.

		Distributionally Robust Optimization (DRO) provides a mature framework for making decisions robust to distributional uncertainty \citep{hanasusanto2015distributionally, sinha2017certifying}. In parallel, model safety research has highlighted the vulnerability of large language models to adversarial prompts and distribution shifts, including jailbreak attacks and transferable adversarial suffixes \citep{wei2023jailbreaking, zou2023universal, yi2024jailbreak}. Our work aims to provide a formal, distributional model for such phenomena. By establishing a theoretical link between model capacity and its intrinsic robustness radius, we offer a principled explanation for empirical observations that larger models may be more resistant to certain types of interference.

		In summary, we draw from theoretical ICL frameworks but focus on adversarial shifts. We extend the work of \citet{ma2025provable} by using the Wasserstein metric and deriving a capacity-dependent robustness bound, connecting DRO theory with ICL to provide a new theoretical foundation for understanding distributional robustness.

		\section{Broader Discussion}
		\label{app:broader-discussion}

		This section discusses several broader aspects of our work. We briefly connect our results to other theoretical perspectives, mention some practical implications, and note limitations that suggest directions for future work.

		Our analysis links to established generalization theory in several ways. The Lipschitz-based mechanism, where robustness scales with $\sqrt{d/m}$, aligns with stability-based generalization bounds. The ridge regression equivalence can be viewed through a PAC-Bayesian lens, interpreting in-context learning as implicit Bayesian inference with a Gaussian prior. Our framework extends this view to adversarial settings by considering worst-case distributions within a Wasserstein ball. The additional sample requirement under perturbation resembles notions of uniform stability, though our bound explicitly quantifies this cost in terms of the perturbation radius $\rho$.

		For practical model design, our results suggest a few guidelines. Architectural capacity, captured by attention dimension $m$, is a primary resource for robustness, scaling as $\sqrt{m}$. This implies diminishing returns: doubling robustness requires quadrupling capacity. Distributing this capacity across multiple attention heads appears beneficial for stability, as suggested by our additional experiments. During training, using a diverse set of pretraining tasks can enlarge the effective robustness region. The regularization parameter $\lambda_N$ presents a trade-off: a stronger prior (higher $\lambda_N$) reduces sensitivity to distribution shift but may lower peak performance. In deployment, one can estimate a plausible perturbation strength $\rho$ from domain context, check if the model's capacity provides sufficient robustness via our bound, and, if not, provision additional in-context examples according to the sample complexity tax.

		Several limitations point to fruitful future research. Our theoretical analysis focuses on linear self-attention Transformers; extending it to standard softmax attention and deep, multi-layer architectures is an important next step. The Gaussian assumptions on task distributions provide tractability but may not hold in all scenarios; analysis for heavy-tailed or discrete distributions would be valuable. Our current framework centers on linear regression tasks; generalizing it to classification and other few-shot learning settings would broaden its applicability. Finally, while our synthetic experiments validate the core scaling laws, testing these relationships on large-scale language models and against real-world adversarial prompts remains an essential empirical challenge.

		In summary, this work provides a distributionally robust foundation for analyzing in-context learning. It formalizes the intuition that model capacity is intrinsically linked to robustness and offers a principled way to reason about performance under adversarial distribution shift.

\bibliographystyle{plainnat}
\bibliography{ref}
\end{document}